\documentclass{scrartcl}
\usepackage[utf8]{inputenc} %
\usepackage[T1]{fontenc}    %
\usepackage{times}
\usepackage{amsfonts}       %
\usepackage{amsmath}
\usepackage{amssymb}
\usepackage{mathtools}
\usepackage{amsthm}
\usepackage{bm}
\usepackage{color}
\usepackage{soul}           %
\usepackage{url}
\usepackage{booktabs}       %
\usepackage{nicefrac}       %
\usepackage{enumerate}
\usepackage{graphicx}
\DeclareGraphicsExtensions{.pdf,.png,.jpg}
\graphicspath{{figures/}}
\usepackage{algorithm}
\usepackage{algpseudocode}
\usepackage{comment}
\usepackage{centernot}     %
\usepackage[numbers]{natbib}
\usepackage{multirow}
\usepackage[titletoc,title]{appendix}

\usepackage{fullpage}
\usepackage[titletoc,title]{appendix}
\pdfoutput=1

\newcommand{\R}{\mathbb{R}}

\newcommand{\BigO}[1]{\ensuremath{\mathcal{O}\left(#1\right)}}                  %
\newcommand{\BigOT}[1]{\ensuremath{\widetilde{\mathcal{O}}\left(#1\right)}}     %
\newcommand{\BigOm}[1]{\ensuremath{\Omega\left(#1\right)}}                      %
\newcommand{\BigT}[1]{\ensuremath{\Theta\left(#1\right)}}                       %

\newcommand{\vect}[1]{\ensuremath{\mathbf{#1}}}                                 %
\newcommand{\vectsym}[1]{\ensuremath{\boldsymbol{#1}}}                          %
\newcommand{\mat}[1]{\ensuremath{\mathbf{\MakeUppercase{#1}}}}                  %
\newcommand{\Exp}[2]{\ensuremath{\mathbb{E}_{#1}\left[#2\right]}}                %
\newcommand{\Var}[2]{\ensuremath{\mathrm{Var}_{#1}\left[#2\right]}}                %
\newcommand{\Cov}[2]{\ensuremath{\mathrm{Cov}_{#1}\left[#2\right]}}              %
\newcommand{\Ind}[1]{\ensuremath{\mathbf{1}\left[#1\right]}}                     %
\newcommand{\Norm}[1]{\ensuremath{\lVert #1 \rVert}}                              %
\newcommand{\NormI}[1]{\ensuremath{\lVert #1 \rVert}_1}                           %
\DeclareMathOperator*{\hadprod}{\circ}                                 		     %
\newcommand{\parg}{\makebox[1ex]{$\mathbf{\cdot}$}}                                %

\newcommand{\matrx}[1]{\begin{bmatrix}#1\end{bmatrix}}                           %

\newcommand{\InNorm}[1]{{\left\vert\kern-0.2ex\left\vert\kern-0.2ex\left\vert #1 
    \right\vert\kern-0.2ex\right\vert\kern-0.2ex\right\vert}}                    %

\newcommand{\InNormII}[1]{{\left\vert\kern-0.2ex\left\vert\kern-0.2ex\left\vert #1 
    \right\vert\kern-0.2ex\right\vert\kern-0.2ex\right\vert}_2}                    %

\newcommand{\InNormInfty}[1]{{\left\vert\kern-0.2ex\left\vert\kern-0.2ex\left\vert #1 
    \right\vert\kern-0.2ex\right\vert\kern-0.2ex\right\vert}_{\infty}}           %

\newcommand{\Abs}[1]{\ensuremath{\lvert #1 \rvert}}                              %
\newcommand{\aAbs}[1]{\ensuremath{\left\lvert #1 \right\rvert}}                  %
\DeclarePairedDelimiterX{\Inner}[2]{\langle}{\rangle}{#1, #2}                    %

\newcommand{\Land}{\wedge}                                                       %
\newcommand{\defeq}{\overset{\mathrm{def}}{=}}                                                      %

\DeclareMathOperator*{\union}{\cup}
\DeclareMathOperator*{\intersection}{\cap}

\DeclareMathOperator*{\notsubseteq}{\nsubseteq}

\DeclareMathOperator*{\argmin}{argmin}

\newtheorem{definition}{Definition}
\newtheorem{proposition}{Proposition}
\newtheorem{assumption}{Assumption}
\newtheorem{lemma}{Lemma}
\newtheorem{theorem}{Theorem}
\newtheorem{remark}{Remark}
\newtheorem{corollary}{Corollary}
\theoremstyle{definition}

\def\independenT#1#2{\mathrel{\rlap{$#1#2$}\mkern2mu{#1#2}}}
\DeclareMathOperator{\independent}{\protect\mathpalette{\protect\independenT}{\perp}}  %
\newcommand{\mcp}[1]{\mathclap{#1}}
\newcommand{\mA}{\mat{A}}
\newcommand{\mB}{\mat{B}}

\newcommand{\mD}{\mat{D}}
\newcommand{\mE}{\mat{E}}
\newcommand{\mI}{\mat{I}}

\newcommand{\mX}{\mat{X}}

\newcommand{\mSig}{\mat{\Sigma}}
\newcommand{\mOmg}{\mat{\Omega}}

\newcommand{\ve}{\vect{e}}

\newcommand{\vv}{\vect{v}}

\newcommand{\vx}{\vect{x}}

\newcommand{\vomg}{\vectsym{\omega}}

\newcommand{\Set}[1]{\{#1\}}    
\newcommand{\Gf}{\mathcal{G}}                  %
\newcommand{\Gfr}{\widetilde{\mathcal{G}}}     %
\newcommand{\Vs}{\mathsf{V}}  %
\newcommand{\Es}{\mathsf{E}}  %
\newcommand{\G}{\mathsf{G}}   %
\newcommand{\Gh}{\widehat{\mathsf{G}}}   %
\newcommand{\Ss}{\mathsf{S}}  %
\newcommand{\Sh}{\widehat{\mathsf{S}}}  %
\newcommand{\Pf}{\mathcal{P}}  %
\newcommand{\Par}[2]{\pi_{#2}(#1)}    %
\newcommand{\Des}[2]{\mathsf{D}_{#2}(#1)}    %
\newcommand{\Anc}[2]{\mathsf{A}_{#2}(#1)}    %
\newcommand{\hPar}[2]{\widehat{\pi}_{#2}(#1)}    %
\newcommand{\Chi}[2]{\phi_{#2}(#1)}   %
\newcommand{\Diag}{\vect{Diag}}
\newcommand{\diag}{\vect{diag}}       
\newcommand{\mi}{\mathsf{-i}}         %
\newcommand{\Sp}{\mathcal{S}}           %
\newcommand{\NB}[2]{\mathsf{N}_{#2}(#1)}            %
\newcommand{\mhOmg}{\widehat{\mOmg}}    %
\newcommand{\mtOmg}{\mOmg^*}            %
\newcommand{\mhB}{\widehat{\mB}}        %
\newcommand{\hB}{\widehat{B}}           %
\newcommand{\homg}{\widehat{\omega}}    %
\newcommand{\Ts}{\mathcal{T}}        %
\newcommand{\varh}{\widehat{\sigma}^2}       %
\newcommand{\var}{\sigma^2}                  %
\newcommand{\hOmega}{\widehat{\Omega}}       %
\newcommand{\Update}{\textsc{Update}}          %
\newcommand{\bomg}{\bar{\omega}}               %
\newcommand{\vbomg}{\bar{\vectsym{\omega}}}    %
\newcommand{\vhomg}{\widehat{\vectsym{\omega}}}    %
\newcommand{\cI}{\mathcal{I}}                  %
\newcommand{\cR}{\mathcal{R}}                  %
\newcommand{\varmax}{\var_{\mathrm{max}}}      %
\newcommand{\varmin}{\var_{\mathrm{min}}}      %
\newcommand{\Bmax}{B_{\mathrm{max}}}           %
\newcommand{\Bmin}{B_{\mathrm{min}}}           %
\newcommand{\alphamax}{\alpha_{\mathrm{max}}}  %
\newcommand{\alphamin}{\alpha_{\mathrm{min}}}  %

\title{Learning linear structural equation models 
in polynomial time and sample complexity}

\author{Asish Ghoshal and Jean Honorio\\
Department of Computer Science\\
Purdue University\\
West Lafayette, IN - 47907\\
\{aghoshal, jhonorio\}@purdue.edu}

\date{}

\begin{document}

\maketitle
\begin{abstract}
The problem of learning structural equation models (SEMs) from data is a fundamental problem in causal inference. We develop a new algorithm --- which is computationally and statistically efficient and works in the high-dimensional regime --- for learning linear SEMs from purely observational data with arbitrary noise distribution. We consider three aspects of the problem: identifiability, computational efficiency, and statistical efficiency.  We show that when data is generated from a linear SEM over $p$ nodes and maximum degree $d$, our algorithm recovers the directed acyclic graph (DAG) structure of the SEM under an \emph{identifiability condition} that is more general than those considered in the literature, and without \emph{faithfulness} assumptions.  In the population setting, our algorithm recovers the DAG structure in $\mathcal{O}(p(d^2 + \log p))$ operations. 
In the finite sample setting, if the estimated precision matrix is sparse, our algorithm has a smoothed complexity of $\BigOT{p^3 + pd^7}$,
while if the estimated precision matrix is dense, our algorithm has a smoothed complexity of $\BigOT{p^5}$.
 For sub-Gaussian noise, we show that our algorithm has a sample complexity of $\mathcal{O}(\frac{d^8}{\varepsilon^2} \log (\frac{p}{\sqrt{\delta}}))$ to achieve $\varepsilon$ element-wise additive error with respect to the true \emph{autoregression matrix} with probability at most $1 - \delta$, while for noise with bounded $(4m)$-th moment, with $m$ being a positive integer, our algorithm has a sample complexity of $\mathcal{O}(\frac{d^8}{\varepsilon^2} (\frac{p^2}{\delta})^{\nicefrac{1}{m}})$.
\end{abstract}
\section{Introduction}
\paragraph{Motivation.}
Elucidating causal relationship between different entities or variables is a fundamental 
task in various scientific disciplines such as finance, genetics, medicine,
neuroscience, artificial intelligence, among others. Learning cause-effect relationships
from purely observational data is often the only recourse available in situations where
performing randomized experiments or interventions can be expensive, impractical, unethical, or downright impossible. 
For continuous-valued variables, structural equation models (SEMs) is a commonly employed formalism
for performing causal inference. Conditions under which SEMs can
be uniquely identified from observational data have been recently characterized. Unfortunately,
for linear SEMs, identifiability conditions have been rather limited, and existing structure 
learning algorithms are inefficient. In this paper, we consider the problem of learning linear 
SEMs over $p$ variables and bounded-degree $d$, from purely observational data, with 
arbitrary noise distributions having bounded second moment
 --- including but not limited to the Gaussian distribution. We generalize existing identifiability
conditions for learning linear SEMs, and present computationally and statistically
efficient algorithms for learning the structure of linear SEMs when identifiable. The paper makes
the following contributions.

\paragraph{Our Contribution.} We present a new identifiability condition for learning linear SEMs from observational data that generalizes the homoscedastic Gaussian noise (equal noise variance) case considered by \cite{Peters2014}. Our algorithm also
works for the case when the noise variances are known up to a constant factor --- a sufficient condition
under which linear SEMs are identifiable as shown by \cite{loh_high-dimensional_2013}.
This disproves an earlier conjecture by \cite{loh_high-dimensional_2013} that "variance scaling or non-Gaussianity is necessary in order to guarantee identifiability" of linear SEMs. Moreover, we show that our identifiability condition is in general necessary for 
ensuring identifiability of linear SEMs, in the sense that there exist an exponential number of DAGs, which under uncountably 
many autogression matrices and noise variances, induce the same covariance and precision matrix, and specify distributions that have the
same conditional independence structures.

Our method is fully non-parametric, works for both Gaussian and non-Gaussian noise, and, to the best of our knowledge, the most efficient algorithm available for learning linear SEMs with provable guarantees. Given the inverse covariance (or precision) matrix, our method, which resembles a Cholesky factorization, can recover the structure and parameters of the SEM exactly in $\mathcal{O}(p(d^2 + \log p))$ floating-point operations.  In the finite sample setting, our method involves estimating the precision matrix, which can be done by solving $p$ linear programs (LPs) and then performing $p$ iterations to learn the structure and parameters of the SEM  by identifying and removing terminal (sink) vertices. If the estimated precision matrix is sparse, then each iteration involves solving at most $d$ linear programs in at most $d^2$ dimensions,  leading to an overall smoothed complexity of $\BigOT{p^3 + pd^7}$. When the estimated precision matrix is dense our method has a smoothed complexity of $\BigOT{p^5}$. This is significantly better than \cite{Peters2014}'s algorithm for learning linear Gaussian SEMs as well as \cite{loh_high-dimensional_2013}'s algorithm for learning SEMs with known noise variance. While the former is is exponential in $p$, the latter is exponential in $d$ and the tree-width of the SEM when the estimated precision matrix is sparse and exponential in $p$ for the dense case.

Our algorithm also works in the high-dimensional regime when $d = o(p)$,
and has a sample complexity of $\mathcal{O}(\frac{d^8}{\varepsilon^2} \log (\frac{p}{\sqrt{\delta}}))$ and  $\mathcal{O}(\frac{d^8}{\varepsilon^2} (\frac{p^2}{\delta})^{\nicefrac{1}{m}})$ for sub-Gaussian noise and noise with bounded $4m$-th moment respectively, for recovering the autoregression matrix of the SEM up to $\epsilon$ additive error with probability at least $1 - \delta$. The sample complexity of our algorithm for sub-Gaussian noise is better than \cite{loh_high-dimensional_2013}'s algorithm, which has a sample complexity of $\BigO{p^2 \log p}$, and is therefore unsuitable for the high-dimensional regime. Moreover, unlike \cite{loh_high-dimensional_2013}'s algorithm, and other methods that use conditional independence tests, for instance, the PC algorithm for learning Gaussian SEMs \cite{kalisch_estimating_2007}, our algorithm does not require any faithfulness conditions, and only requires a weaker \emph{causal minimality} condition. The faithfulness assumption requires that the distribution $\Pf(X)$ contain only those conditional independence assertions that are implied by the \emph{d-separation} criteria of the DAG \cite{spirtes2000causation}.
However, faithfulness cannot be tested from data in full generality \cite{Zhang2008} and algorithms that infer the DAG structure from
a finite number of samples must require \emph{strong faithfulness} \cite{zhang2002strong}, which is a restrictive assumption.
Our results has the following significant yet hitherto known implication for learning Gaussian Bayesian networks. Given data generated from a Gaussian Bayesian network that is causal minimal to the true DAG structure, one can recover the DAG structure in polynomial time and sample complexity from a finite number of samples, under more general identifiability conditions than homoscedastic noise. 

Lastly, we obtain several useful results about the theory of linear SEMs en route to developing our main algorithm for learning linear SEMs.

\paragraph{Our Techniques.} Our algorithm for learning linear SEMs differs conceptually
from previous test-based, score-based or inverse-covariance-estimation-based methods. 
We are therefore able to get rid of many of the shortcomings of existing methods like requirement
of strict non-Gaussianity of noise \cite{shimizu_directlingam:_2011}, homoscedasticity \cite{Peters2014},
and faithfulness \cite{loh_high-dimensional_2013,kalisch_estimating_2007}. We do so my obtaining and 
exploiting various properties of terminal vertices in linear SEMs. We obtain our sample complexity results
by using various properties of sub-Gaussian and bounded-moment variables and using concentration results
for the empirical covariance matrix under the aforementioned noise conditions. Lastly, we improve
the computational complexity of our algorithm by exploiting the sparsity structure of the precision matrix
to obtain solutions of ``larger'' LPs (size $\BigO{p}$) by solving much ``smaller'' LPs (size $\BigO{d^2}$).
\section{Related Work}
We start our discussion of existing literature by first presenting known
identifiability conditions for learning SEMs and Bayesian networks. \cite{peters_causal_2014} proved identifiability of distributions
drawn from a restricted SEM with additive noise, where in the restricted SEM the functions are assumed to be non-linear and thrice continuously differentiable. Linear SEMs are identifiable if (a) the noise variables are non-Gaussian \cite{Shimizu2006}, (b) the
noise variances are known up to a constant factor \cite{loh_high-dimensional_2013}, and (c) noise variables are Gaussian
and have the same variance \cite{Peters2014} (homoscedastic noise). \cite{park_learning_2017} introduced Quadratic Variance Function (QVF)
DAG models --- a class of Bayesian networks in which the conditional variance of a variable is a quadratic function of its conditional mean ---
and proved identifiability of the models from observational data. However, QVF DAG models cannot be expressed as SEMs, and the
quadratic variance property holds for a handful of conditional distributions which includes Binomial, Poisson, Exponential, Gamma, 
and a few others. 

The computational and statistical complexity landscape of learning linear SEMs is peppered by inefficient algorithms.
This is in part justified by various hardness results known in the literature for learning DAGs from observational data
\cite{chickering1996learning,dasgupta1999learning}. Algorithms for learning DAGs can be divided into two categories:
independence test based methods and score based methods. Score based methods use a score function, typically penalized
log-likelihood, to find the best scoring DAG among the space of all DAGs. Since the number of DAGs and degree-bounded
DAGs is exponential in $p$ \cite{Robinson1977,ghoshal2016information}, score-based methods are exponential time.
A popular score function for learning Gaussian SEMs is the $\ell_0$-penalized Gaussian log-likelihood score 
proposed by \cite{van_de_geer_l0-penalized_2013}. 
\cite{Peters2014} proposed using $\ell_0$-penalized Gaussian log-likelihood score for learning homoscedastic noise linear Gaussian SEMs
along with a heuristic greedy search algorithm which is not guaranteed to find the correct (highest-scoring) solution.
\cite{loh_high-dimensional_2013} showed that under a faithfulness assumption, the sparsity pattern of the precision matrix corresponds
to the edge structure of the \emph{moral graph} of the underlying DAG. They exploit this property to devise an algorithm
that searches for the highest-scoring DAG, using dynamic programming, that has the same moral graph as that given by the sparsity
pattern of the precision matrix. Independence test based methods on the other hand require
restrictive faithfulness conditions to guarantee structure recovery. \cite{kalisch_estimating_2007} proposed
using the PC algorithm to learn Gaussian SEMs, which has a computational complexity of $\BigO{p^d}$ and
is only efficient for learning very sparse Gaussian SEMs.
Among computationally efficient algorithms, the \emph{Direct-LiNGAM} algorithm \cite{shimizu_directlingam:_2011},
which strictly requires non-Gaussianity of the noise variables, needs an infinite number
of samples to guarantee structure recovery. This is because of the use of independence testing between a variable and its residuals
to detect exogenous variables (variables with no parents). For the same reason, the correctness of 
\emph{RESIT} \cite{peters_causal_2014}, which is a computationally efficient algorithm for learning \emph{non-linear SEMs}, 
is only guaranteed in the population setting.

Other authors have proposed various approximation algorithms and heuristic methods for learning Bayesian networks,
which can be used to learn Gaussian SEMs by using appropriate score functions. Popular heuristic methods are
max-min hill climbing (MMHC) algorithm by \cite{tsamardinos2006max}, and the Greedy Equivalence Search (GES) algorithm 
proposed by \cite{chickering_optimal_2003}. \cite{jaakkola_learning_2010}
proposed an LP-relaxation based method for learning Bayesian networks which is an approximation algorithm.
\section{Preliminaries}
We begin this section by introducing our notations and definitions before
formalizing the problem of learning linear SEMs from observational data. 
We will let $[p] \defeq \Set{1, \ldots, p}$. Vectors and matrices are
denoted by lowercase and uppercase bold faced letters respectively. 
Random variables (including random vectors) are denoted by uppercase letters. For any two non-empty index sets 
$s_r, s_c \subseteq [p]$, the matrix $\mA_{s_r, s_c} \in \R^{\Abs{s_r} \times \Abs{s_c}}$
denotes the submatrix of $\mA \in \R^{p \times p}$ obtained by selecting the
$s_r$ rows and $s_c$ columns of $\mA$. With a slight abuse of notation, we will allow the 
index sets $s_r$ and $s_c$ to be a single index, e.g., $i$,
and we will denote the index set of all rows (or columns) by $*$.
For any matrix $\mA$ (equivalently for vectors), we will denote its 
support set by: $\Sp(\mA) = \Set{(i,j) \in [p] \times [p] \,|\, A_{i,j} \neq 0 }$.
Vector $\ell_p$ norms are denoted by $\Norm{\parg}_p$. 
For matrices, $\Norm{\parg}_p$ denotes the induced (or operator) $\ell_p$-norm
and $\Abs{\parg}_p$ denotes the elementwise $\ell_p$ norm, i.e., 
$\Abs{\mA}_p \defeq (\sum_{i,j} \Abs{A_{i,j}}^p)^{\nicefrac{1}{p}}$.
For two matrices $\mA$ and $\mB$, $\mA \hadprod \mB$ denotes the Hadamard product of $\mA$ and $\mB$,
while $\diag(\mA)$ denotes the vector formed by taking the diagonal of $\mA$. For a
vector $\vv$, $\Diag(\vv)$ denotes the diagonal matrix with $\vv$ in the diagonal. 
Finally, we define the set $\mi \defeq [p] \setminus \Set{i}$.

Let $\G = ([p], \Es)$ be a directed
acyclic graph (DAG) where $[p]$ is the vertex set and $\Es \subset [p] \times [p]$ is the set of directed edges.
An edge $(i,j) \in \Es$ implies the edge $i \leftarrow j$.
We denote by $\Par{i}{\G}$ and $\Chi{i}{\G}$ the parent set and the set of children of the $i$-th node respectively,
in the graph $\G$; and drop the subscript $\G$ when the clear from context. The set of neighbors
of the $i$-th node is denoted by $\NB{i}{\G} = \Par{i}{\G} \union \Chi{i}{\G}$. A node $j$ is a \emph{descendant}
of $i$ in $\G$ if there exists a (directed) path from $i$ to $j$ in $\G$. We will denote the set of descendants of
$i$ by $\Des{i}{\G}$. Similarly, we will denote the set of ancestors of $i$ --- nodes $j$ such that there is a
path from $j$ to $i$ in $\G$ --- by the set $\Anc{i}{\G}$.
A vertex $i \in [p]$ is a \emph{terminal vertex} in $\G$ if $\Chi{i}{\G} = \varnothing$. For each $i \in [p]$ we have
a random variable $X_i \in \R$, $X = (X_1, \ldots, X_p) \in \R^p$ is the $p$-dimensional vector of random variables,
and $\vx = (x_1, \ldots, x_p)$ is a joint assignment to $X$. 
Every DAG $\G = ([p], \Es)$ defines a set of
topological orderings $\Ts_{\G}$ over $[p]$ that are compatible with the DAG $\G$, i.e.,
$\Ts_{\G} = \Set{\tau \in \mathrm{S}_p \mid \tau(j) < \tau(i) \text{ if } (i,j) \in \Es}$, 
where $\mathrm{S}_p$ is the set of all possible permutations of $[p]$.

The random vector $X$ follows a linear structural equation model (SEM), if each variable can be written
as a linear combination of the variables in its parent set as follows:
\begin{align}
X_i = \sum_{j \in \Par{i}{\G}} B_{i,j} X_j + N_i && (\forall i \in [p]), \label{eq:sem}
\end{align}
where $\G = ([p], \Es)$ is a DAG, 
$N = (N_1, \ldots, N_p)$ are the noise variables, and $N_i \independent X_{1}, \ldots, X_{i - 1}$.
Without loss of generality, we assume that $\Exp{}{X_i} = \Exp{}{N_i} = 0,\, \forall i \in [p]$.
As is typically the case in the literature of SEMs,
we further assume that the noise variables $N_i$ have bounded second moments and are independent.
Thus $\Cov{}{N} = \Exp{}{N N^T} = \Diag(\var_1, \ldots, \var_p)$. We can then write \eqref{eq:sem}
in vector form as follows:
\begin{align}
X = \mB X + N, \label{eq:sem_vect}
\end{align}
where $\mB = (B_{i,j})$ is referred to as the \emph{autoregression matrix} and $\Sp(\mB) = \Es$.
Therefore, we will denote an SEM by the triple $(\G, \mB, \Set{\var_i}_{i \in [p]})$, or more compactly by $(\G, \mB, \Set{\var_i})$.

Given an SEM $(\G, \mB, \Set{\var_i})$, the joint distribution $\Pf(X)$ is completely determined and
factorizes according to the DAG structure $\G$:
\begin{align}
\Pf(X; \G) = \prod_{i=1}^p \Pf_i(X_i | X_{\Par{i}{\G}}; \G),   \label{eq:bn}
\end{align}
where $\Pf_i$ is the conditional distribution of the $X_i$. We then say that the distribution $\Pf$
is \emph{Markov with respect to the DAG $G$}, i.e., $X_i$ satisfies
the Markov condition: $X_i \independent X_{j} \mid X_{\Par{i}{}}, \,
\forall i \in [p], \forall j \in [p] \setminus (\Des{i}{} \union \Par{i}{} \union \Set{i})$.
Thus an SEM is equivalent to a \emph{Bayesian network}.
Specifically, if the noise variables are Gaussian, then $\Pf$ is a \emph{Gaussian Bayesian network} (GBN),
where the joint distribution $\Pf$ and the conditional distributions $\Pf_i$ are Gaussian. 
We obtain our theoretical results for the class of degree-bounded DAGs 
$\Gf_{p,d} \defeq \Set{\G \mid \G = ([p], \Es) \text{ is a DAG and } \Abs{\NB{i}{\G}} \leq d,\, \forall i \in [p]}$.

Next, we define the notion of \emph{causal minimality}, introduced by \cite{Zhang2008}, which is important 
for ensuring identifiability of linear SEMs considered in this paper.
\begin{definition}[Causal Minimality] Given a DAG $\G$, a distribution $\Pf(X)$, that is Markov with respect to $G$,
is causal minimal if $\Pf$ is not Markov with respect to a proper subgraph of $\G$.
\end{definition}
Our assumption of $\Sp(\mB) = \Es$, ensures that Lemma 4 of \cite{peters_causal_2014} holds
for all SEMs $(\G, \mB, \Set{\var_i})$. This in turn implies that 
the joint distribution $\Pf(X)$ determined by the SEM $(\G, \mB, \Set{\var_i})$
is causal minimal with respect to $\G$ (see Proposition 2 in \cite{peters_causal_2014}).
Therefore, the SEMs considered in the paper are causal minimal.

The problem of learning the structure of an SEM is as follows.
Given an $n \times p$ data matrix $\mX = (\vx_1, \ldots, \vx_p)$, with $\vx_i \in \R^n$,
drawn from an SEM $(\G^*, \mB^*, \Set{\var_i})$ with $\G^* \in \Gf_{p,d}$, we want to learn an SEM $(\Gh, \mhB, \Set{\varh_i})$
from $\mX$ such that $\G^* = \Gh$.

\section{Learning SEMs with unknown error variances}
We start with presenting our main results for learning SEMs when the error variances
are unknown. Our algorithm for learning SEMs works by constructing the SEM in a bottom-up fashion.
The algorithm has $p$ iterations. In each iteration it identifies and removes a terminal vertex,
learning its parent set and edge weights along the way. We
show that, under a certain identifiability condition which generalizes other identifiability 
conditions known in the literature, e.g., homoscedastic errors, and without assuming faithfulness of the distribution to the DAG,
each of these steps can be performed efficiently using only the precision matrix or
an estimator of it. 
\subsection{Identifiability}
The following assumption gives a sufficient condition under which the structure and parameters
of an SEM can be uniquely recovered from observational data using Algorithm \ref{alg:population}.
The assumption is defined in terms of subgraphs of $\G$ obtained
by removing terminal vertices sequentially. For any $\tau \in \Ts_{\G}$,
we will consider sequence of graphs $\G[m,\tau] = (\Vs[m, \tau], \Es[m, \tau])$,
indexed by $(m, \tau)$, where $\G[m, \tau]$ is
the induced subgraph of $\G$ over the first $m$ vertices in the topological ordering $\tau$,
i.e., $\Vs[m, \tau] \defeq \Set{i \in [p] \mid \tau(i) \leq m}$ and
$\Es[m, \tau] \defeq \Set{(i,j) \in \Es \mid i \in \Vs[m, \tau] \Land j \in \Vs[m, \tau]}$.
\begin{assumption}[Identifiability condition]
\label{ass:identifiability_condition}
Given an SEM $(\G, \mB, \Set{\var_i})$ with $\G \in \Gf_{p,d}$, then
$\forall (i, j) \in \Vs[m, \tau] \times \Vs[m, \tau], m \in [p]$, and $\tau \in \Ts_{\G}$,
such that $\Chi{i}{\G[m,\tau]} = \varnothing \Land \Chi{j}{\G[m,\tau]} \neq \varnothing$:
\begin{align}
\frac{1}{\sigma_i^2} < \frac{1}{\sigma_j^2} + \sum_{l \in \Chi{j}{\G[m, \tau]}} \frac{B_{l,j}^2}{\sigma_l^2} \label{eq:identifiability_condition},
\end{align}
\end{assumption}
As we will show later, Assumption \ref{ass:identifiability_condition} essentially lays down a condition
under which terminal vertices, and subsequently the causal order, can be identified from the precision matrix. 
From Assumption \ref{ass:identifiability_condition}, we immediately get the following special cases for
identifiability of linear SEMs, where the first one is the homoscedastic case known in the literature,
while the second case is new.
\begin{proposition}[Sufficient conditions for identifiability]
\label{prop:suff_cond_identifiability}
Let $(\G, \mB, \Set{\var_i})$ be an SEM satisfying Assumption \ref{ass:identifiability_condition},
with precision matrix $\mOmg$.
Then, either of the following two conditions are sufficient for 
uniquely identifying the autoregression matrix  $\mB$ and the DAG $\G$ from $\mOmg$:
\begin{enumerate}[(i)]
\item $\forall i \in [p], \, \sigma_i = \sigma$, for some $\sigma > 0$,
\item $\forall i \in [p], \, 1 < \sigma_i \leq \Bmin$, where $\Bmin \defeq \min \Set{\Abs{B_{i,j}} \mid (i,j) \in \Es}$.
\end{enumerate}
\end{proposition}
Next we show that the identifiability condition \ref{ass:identifiability_condition} is in general necessary,
i.e, if Assumption \ref{ass:identifiability_condition}  is violated, then there exists an 
exponential number of DAG structures that, coupled with an uncountable number of autoregression
matrices and noise variances, induce the same covariance and precision matrix, and determine joint distributions $\Pf(X)$
that are causal minimal and Markov to the DAG structures. \emph{Therefore, no algorithm based on independence testing,
or that uses solely the covariance or precision matrix, can recover the true DAG structure in polynomial time}.
In the following lemma we will equivalently denote an SEM by $(\G, \mB, \mD)$
where $\mD$ is a diagonal matrix with $D_{i,i} = \var_i$.
\begin{lemma}
\label{lemma:necessary}
There exists $\Gfr_{p,d} \subset \Gf_{p,d}$ with $\Abs{\Gfr_{p,d}} = 2^{\BigT{p}}$,
autoregression matrices $\mB(\beta)$ parameterized by $\beta$, and diagonal matrices $\mD(v_1, v_2)$
parameterized by $v_1, v_2$ such that any SEM $(\G, \mB(\beta), \mD(v_1, v_2))$ with $\G \in \Gfr_{p,d}$
does not satisfy Assumption \ref{ass:identifiability_condition}, induces the same covariance and precision matrix,
and distribution $\Pf(X)$ that has the same conditional independence structure, $\forall \beta \in (-\infty, \infty)$,
$v_1 \in (0, \infty)$ and $v_2 > v_1$.
\end{lemma}
Next, we present a series of results building towards our main result for learning SEMs from
precision matrix. In the following proposition we characterize the precision matrix of linear SEMs.
\begin{proposition}
\label{prop:precision}
Let $(\G, \mB, \Set{\var_i})$ be an SEM over $X$, then the precision matrix is
given as: $\mOmg = (\mI - \mB)^T \mD^{-1} (\mI - \mB)$, where $\mD = \Diag(\var_1, \ldots, \var_p)$.
The entries of the precision matrix is given as:
\begin{align}
\Omega_{i,i} = \frac{1}{\var_i} + \sum_{l \in \Chi{i}{}} \frac{B_{l,i}^2}{\var_l}, &&
\Omega_{i,j} = - \frac{B_{i,j}}{\var_i} - \frac{B_{j,i}}{\var_j} 
	+  \sum_{l \in \Chi{i}{} \intersection \Chi{j}{}} \frac{B_{l,i} B_{l,j}}{\var_l}. \label{eq:precision_mat}
\end{align}
\end{proposition}
The above characterization of the precision matrix motivates our 
indentifiability condition given by Assumption \ref{ass:identifiability_condition},
and also provides a recipe for identifying terminal vertices from the precision matrix
as is formalized by the following proposition.
\begin{proposition}
\label{prop:terminal_vertex}
Let $(\G, \mB, \Set{\var_i})$ be a SEM over $X$ with precision matrix $\mOmg$,
that satisfies the identifiability condition given by Assumption \ref{ass:identifiability_condition}.
Then, $i$ is a terminal vertex in $\G$ if and only if $i \in \argmin(\diag(\mOmg))$.  Further, if $i$
is a terminal vertex then $\var_i = \nicefrac{1}{\Omega_{i,i}}$.
\end{proposition}
The next proposition, which follows directly from Proposition \ref{prop:terminal_vertex} and \eqref{eq:precision_mat},
states that for a terminal vertex the parent set and edge weights can be conveniently ``read off'' from the precision matrix.
This is the key result which helps us avoid the faithfulness condition.
\begin{proposition}
\label{prop:weights}
Let $(\G, \mB, \Set{\var_i})$ be an SEM over $X$ with precision matrix $\mOmg$.
If $i$ is a terminal vertex in $\G$, then $\mB_{i,*} = - \nicefrac{\mOmg_{i,*}}{\Omega_{i,i}}$
and $\Par{i}{\G} = \Sp(\mOmg_{i,*}) \setminus \Set{i}$.
\end{proposition}
The following lemma is a useful result about linear SEMs with arbitrary noise distribution, that generalizes
a result so far known only for the Gaussian distribution --- for a terminal vertex $i$, the precision matrix 
over $X_{\mi}$ can be obtain by performing a Schur complement update of the precision matrix over $X$.
While, the result for the Gaussian distribution holds for all variables, the analogous result for 
general SEMs holds only for terminal vertices.
\begin{lemma}
\label{lemma:schur_update}
Let $(\G, \mB, \Set{\var_i})$ be an SEM over $X$ with precision matrix $\mOmg$. Let $i$
be a terminal vertex in the $\G$, then the precision matrix over $X_{\mi}$, $\mOmg_{(\mi)}$, is given as:
\begin{align*}
\mOmg_{(\mi)} = \mOmg_{\mi, \mi} - \frac{1}{\Omega_{i,i}} \mOmg_{\mi,i} \mOmg_{i,\mi}.
\end{align*}
\end{lemma}
Finally, the following lemma characterizes the entries of the precision matrix over $X_{\mi}$
and will be very useful in developing our finite-sample algorithm for learning SEMs.
\begin{lemma}
\label{lemma:precision_minus_i}
Let $(\G, \mB, \Set{\var_i})$ be a SEM over $X$ with precision matrix $\mOmg$. Let $i$
be a terminal vertex in the $\G$ and let $\mOmg_{(\mi)}$ denote the precision matrix over $X_{\mi}$.
Then, 
\begin{align*}
(\mOmg_{(\mi)})_{j,k} &= \mOmg_{j,k} && (\forall (j,k) \in \mi \times \mi \mid \Set{j,k} \notsubseteq \Par{i}{\G}), \\
\Sp((\mOmg_{(\mi)})_{j,*}) &\subseteq (\Sp(\mOmg_{j,*}) \setminus \Set{i}) \union \Par{i}{\G} && (\forall j \in \Par{i}{\G}).
\end{align*}
\end{lemma}
With the required results in place, we are now ready to present our main algorithm, detailed in Algorithm \ref{alg:population},
for learning SEMs from the precision matrix. The role of the diagonal matrix $\mD$ will become clear 
in the next section where we focus on the problem  of learning SEMs with known error variances. 
For now we simply set $\mD$ to the identity matrix $\mI$. The following theorem proves the correctness of
our algorithm in the population setting.
\begin{figure*}[htbp]
\begin{minipage}{0.4\linewidth}
\begin{algorithm}[H]
\caption{SEM structure learning algorithm. \label{alg:population}}
\begin{algorithmic}[1]
\Require Precision matrix $\mOmg$, diagonal matrix $\mD$.
\Ensure $\Gh, \mhB$.
\State $\mhB \gets \vect{0}$.
\For{$t \in [p]$}
	\State $i \gets \argmin(\diag(\mOmg \hadprod \mD))$. \label{line:terminal_vertex}
	\State $\mB_{i, *} \gets -\nicefrac{\mOmg_{i,*}}{\Omega_{i,i}}$, $B_{i,i} \gets 0$. \label{line:pop_ith_weights}
	\State $\mOmg \gets \mOmg - \frac{1}{\Omega_{i,i}} \mOmg_{*, i} \mOmg_{i, *}$. \label{line:pop_update}
	\State $\Omega_{i,i} \gets \infty$.
\EndFor 
\State $\Gh \gets ([p], \Sp(\mhB))$.
\end{algorithmic}
\end{algorithm}
\end{minipage}\hfill%
\begin{minipage}{0.5\linewidth}
\begin{algorithm}[H]
\caption{Updating a precision matrix, after removing
a terminal vertex, using CLIME. \label{alg:clime_update}}
\begin{algorithmic}[1]
\Function{Update}{$\mhOmg, i, \lambda_n$}
\State $\hPar{i}{} \gets \Sp(\mhOmg_{i,*}) \setminus \Set{i}$.
\For{$j \in \hPar{i}{}$}
\State $\Sh_j \gets \left(\Sp(\mhOmg_{j,*}) \setminus \Set{i}\right) \union \hPar{i}{}$.
\State Compute $\vbomg_{j}$ by solving \eqref{eq:omegahat_colwise} for $\mSig^n_{\Sh_j, \Sh_j}$.
\State $\mhOmg_{j, \Sh_j} = \mhOmg_{\Sh_j, j} \gets \vbomg_j$
\EndFor
\State $\mhOmg_{i,*} \gets \vect{0}$ and $\mhOmg_{*,i} \gets \vect{0}$.
\State \Return $\mhOmg$.
\EndFunction
\end{algorithmic}
\end{algorithm}
\end{minipage}
\end{figure*}
\begin{theorem}
\label{thm:main_unknown_error}
Let $(\G, \mB, \Set{\var_i})$ be an SEM over $X$, with precision matrix $\mOmg$,
satisfying Assumption \ref{ass:identifiability_condition}. Then, given $(\mOmg, \mI)$ as input, 
Algorithm \ref{alg:population} returns a unique $(\Gh, \mhB)$ such that $\Gh = \G$ and $\mhB = \mB$.
\end{theorem}
As a consequence of the above theorem we have the following corollary about identifiability of linear SEMs.
\begin{corollary}
An SEM $(\G, \mB, \Set{\var_i})$ satisfying Assumption \ref{ass:identifiability_condition}
is identifiable, and can be uniquely identified from the precision matrix $\mOmg$.
\end{corollary}
\subsection{Statistical guarantees for estimation}
Algorithm \ref{alg:population} can be used to learn a SEM given an estimate of the precision
matrix, computed from a finite number of samples, with a slight modification. In line \ref{line:pop_update}
instead of using the Schur complement update, we use Algorithm \ref{alg:clime_update} to update
the precision matrix after a terminal vertex has been identified (and removed). The rationale behind
this is that even if the estimated precision matrix is close to the true precision matrix, 
the Schur updates could still result in errors accumulating in the precision matrix. In order to
ensure that our algorithm is statistically efficient, we need more control over those errors,
which in turns calls for some sort of penalization for estimating from a finite number of samples.
\paragraph{Inverse covariance matrix estimation.}
A key step of our algorithm is estimating the inverse covariance matrix
over $X$ or a subset of $X$.  Due in part to its role in undirected graphical model selection, 
the problem of inverse covariance matrix estimation has received significant attention over the years.
A popular approach for inverse covariance estimation, under high-dimensional settings, 
is the $\ell_1$-penalized Gaussian maximum likelihood estimate (MLE) studied by 
\cite{yuan2007model}, \cite{banerjee2008model}, and \cite{friedman_sparse_2008}, among others.
The $\ell_1$-penalized Gaussian MLE estimate of the inverse covariance matrix has
attractive theoretical guarantees as shown by \cite{ravikumar_high-dimensional_2011}. 
However, the elementwise $\ell_{\infty}$ guarantees for the 
inverse covariance estimate obtained by \cite{ravikumar_high-dimensional_2011}
require an edge-based mutual incoherence condition that is quite restrictive. 
Many algorithms have been developed in the recent past for solving the $\ell_1$-penalized 
Gaussian MLE problem \cite{hsieh_big_2013, hsieh2012divide, rolfs2012iterative, johnson2012high}. 
While, technically, these algorithms can be  used in conjunction with our algorithm for learning SEMs, 
in this paper we use the method called CLIME, developed by \cite{cai_constrained_2011}. 
The primary motivation behind using CLIME is that the theoretical guarantees obtained by  
\cite{cai_constrained_2011} does not require the edge-based mutual incoherence condition. 
Further, CLIME is computationally attractive because it 
computes $\mhOmg$ columnwise by solving $p$ independent linear programs. 
Even though the CLIME estimator $\mhOmg$ is not guaranteed to be positive-definite
(it is positive-definite with high probability) it is suitable for our purpose. 
Next, we briefly describe the CLIME method for inverse covariance estimation and 
instantiate the theoretical results of \cite{cai_constrained_2011} for our purpose.

The CLIME estimator $\mhOmg$ is obtained as follows. First, we compute 
a potentially non-symmetric estimate $\bar{\mOmg} = (\bomg_{i,j})$ 
by solving the following:
\begin{align}
\bar{\mOmg} = \argmin_{\mOmg \in \R^{p \times p}} \Abs{\mOmg}_1 
	\text{ s.t. } \Abs{\mSig^n \mOmg - \mI}_{\infty} \leq \lambda_n, \label{eq:omegahat}
\end{align}
where $\lambda_n > 0$ is the regularization parameter, $\mSig^n \defeq (\nicefrac{1}{n}) \mX^T \mX$ is the empirical covariance matrix, and $\Abs{\cdot}_{1}$ (respectively $\Abs{\cdot}_{\infty}$) denotes elementwise $\ell_1$ (respectively $\ell_{\infty}$) norm. Finally, the symmetric estimator is obtained by selecting the smaller entry among $\bomg_{i,j}$ and $\bomg_{j,i}$, i.e.,
$\mhOmg = (\homg_{i,j})$, where 
$\homg_{i,j} = \bomg_{i,j} \Ind{\Abs{\bomg_{i,j}} < \Abs{\bomg_{j,i}}} + \bomg_{j,i} \Ind{\Abs{\bomg_{j,i}} \leq \Abs{\bomg_{i,j}}}$.
It is easy to see that \eqref{eq:omegahat} can be decomposed into $p$ linear programs as follows. 
Let $\bar{\mOmg} = (\vbomg_1, \ldots, \vbomg_p)$, then
\begin{align}
\vbomg_i = \argmin_{\vomg \in \R^p} \NormI{\vomg} \text{ s.t. } \Abs{\mSig^n \vomg - \ve_i}_{\infty} \leq \lambda_n, 
\label{eq:omegahat_colwise}
\end{align}
where $\ve_i = (e_{i,j})$ such that $e_{i,j} = 1$ for $j = i$ and $e_{i,j} = 0$ otherwise. 
The main result about the CLIME estimator that we use from \cite{cai_constrained_2011} is 
given by the following lemma, which is a minor reformulation of Theorem 6 in \cite{cai_constrained_2011}:
\begin{lemma}[\cite{cai_constrained_2011}]
\label{lemma:clime}
Let $(\G, \mB, \Set{\var_i})$ be an SEM over $X$, with
covariance and precision matrix $\mSig$ and $\mOmg$ respectively. 
Let $\mhOmg$ be the estimator of $\mOmg$ obtained by solving the 
optimization problem given by \ref{eq:omegahat_colwise}. Then if $\lambda_n \geq \NormI{\mOmg} \Abs{\mSig - \mSig^n}_{\infty}$,
then $\Abs{\mOmg - \mhOmg}_{\infty} \leq 4 \NormI{\mOmg} \lambda_n$. Further, if
\begin{align*}
\min \Set{ \Abs{\Omega_{i,j}}  \mid (i,j) \in [p] \times [p] \Land \Abs{\Omega_{i,j}} \neq 0} > 4 \NormI{\mOmg} \lambda_n, 
\end{align*} 
then $\Sp(\mOmg) \subseteq \Sp(\mhOmg)$.
\end{lemma}
Next we state out finite sample identifiability condition. This differs from the population
version in that we require a ``gap'' between the diagonal entries of the precision matrix
for terminal and non-terminal vertices. This gap, as we show later, 
must scale as $\BigOm{d^2 \sqrt{\frac{\log p}{n}}}$ and $\BigOm{\frac{d^2  (p)^{\nicefrac{1}{m}}}{\sqrt{n}}}$
for sub-Gaussian noise and bounded moment noise respectively. Condition (ii) of the below assumption
also restricts how fast the ``minimum'' non-diagonal entry of the precision matrix must decay. 
Note that our conditions are weaker than those of \cite{loh_high-dimensional_2013} due to which
we are able to achieve better sample complexity than their algorithm.
\begin{assumption}[Finite Sample Identifiability Condition]
\label{ass:variance_condition}
Let $(\G, \mB, \Set{\var_i})$ be an SEM with  inverse covariance matrix $\mOmg$. 
Let $\mOmg_{(m, \tau)}$ denote the inverse covariance matrix over $X_{\Vs[m, \tau]}$, and
\begin{align}
M \defeq \max\Set{\NormI{\mOmg_{(m, \tau)}} \mid m \in [p], \tau \in \Ts_{\G}}. \label{eq:constant_M}
\end{align}
Then, we have that
\begin{enumerate}[(i)]
\item $\forall (i, j) \in \Vs[m, \tau] \times \Vs[m, \tau], m \in [p]$, and $\tau \in \Ts_{\G}$, 
such that $\Chi{i}{\G[m,\tau]} = \varnothing \Land \Chi{j}{\G[m,\tau]} \neq \varnothing$:
\begin{align*}
\frac{1}{\sigma_i^2} < \frac{1}{\sigma_j^2} + \sum_{l \in \Chi{j}{\G[m, \tau]}} \frac{B_{l,j}^2}{\sigma_l^2} - 8 M \lambda_n,
\end{align*}
\item $\min \Set{\Abs{(\Omega_{(m, \tau)})_{i,j}} \mid (\Omega_{(m, \tau)})_{i,j} \neq 0, (i, j) \in \Vs[m, \tau] \times \Vs[m, \tau],
 m \in [p], \tau \in \Ts_{\G}} > 4 M \lambda_n$,
\item for all $i \in [p]$, $\var_i \in o(\nicefrac{1}{4 M \lambda_n})$.
\end{enumerate}
\end{assumption}
The following lemma proves the correctness of Algorithm \ref{alg:clime_update} which
updates the precision matrix, after removing a terminal vertex.
\begin{lemma}
\label{lemma:clime_update}
Let $(\G, \mB, \Set{\var_i})$ be an SEM over $X$ with precision matrix $\mOmg$. 
Let $\mhOmg$ be an estimator of $\mOmg$ such that $\Abs{\mOmg - \mhOmg}_{\infty} \leq 4 M \lambda_n$,
and $\Sp(\mOmg) \subseteq \Sp(\mhOmg)$, where $M$ is defined in \eqref{eq:constant_M}.
Let $i$ be a terminal vertex in the $\G$, $\mOmg_{(\mi)}$ be the true precision matrix over $X_{\mi}$, 
and let $\mhOmg'$ be the matrix returned by the function \Update{}. 
Then, $\Abs{\mOmg_{(\mi)} - \mhOmg'_{\mi, \mi}}_{\infty} \leq 4 M \lambda_n$ and $\Sp(\mOmg_{(\mi)}) \subseteq \Sp(\mhOmg')$.
\end{lemma}
\begin{theorem}
\label{thm:main_finite_sample}
Let $(\G^*, \mB^*, \Set{\var_i})$ be the true SEM,
with covariance and precision matrix $\mSig^*$ and $\mtOmg$, respectively,
from which a data set $\mX$ of $n$ samples is drawn. If the regularization parameter
satisfies $\lambda_n \geq M \Abs{\mSig^n - \mSig^*}$, then under Assumption \ref{ass:variance_condition},
the Algorithm \ref{alg:population}, with $\mD$ set to $\mI$, returns an estimator
$\mhB$ such that $\Abs{\mB^* - \mhB} \leq c4 M (1 + \Bmax) \varmax \lambda_n$, $\Sp(\mB^*) \subseteq \Sp(\mhB)$,
and $\Ts_{\Gh} \subseteq \Ts_{\G^*}$, where $c \leq \nicefrac{\varmin}{(1 - 4 M \lambda_n \varmin)}$ is a constant.
\end{theorem}
\begin{theorem}[Sub-Gaussian noise]
\label{thm:sub_gaussian}
Given an SEM $(\G^*, \mB^*, \Set{\var_i})$ with $\G^* \in \Gf_{p,d}$ satisfying Assumptions \ref{ass:variance_condition}
such that $\nicefrac{N_i}{\sigma_i}$ is sub-Gaussian with parameter $\nu$;
if the regularization parameter and number of samples satisfy the following conditions:
\begin{align*}
\lambda_n \geq M C_1 \sqrt{\frac{2}{n} \log \left(\frac{2p}{\sqrt{\delta}}\right) }, &&
n \geq \frac{2(c C_1 4 M^2 (1 + \Bmax) \varmax)^2}{\varepsilon^2} \log \left(\frac{2p}{\sqrt{\delta}}\right),
\end{align*}
then $\Abs{\mB^* - \mhB}_{\infty} \leq \varepsilon$ with probability at least $1 - \delta$,
where $C_1 = \sqrt{128}(1 + 4\nu^2) (\max_i \mSig^*_{i,i})$, $c$ is defined in Theorem \ref{thm:main_finite_sample},
and $M$ is given by \eqref{eq:constant_M}. Further, thresholding $\mhB$ at the level $\varepsilon$
we get that $\Sp(\mhB) = \Sp(\mB^*)$ and $\Gh = \G^*$.
\end{theorem}
\begin{theorem}[Bounded moment noise]
\label{thm:bounded_moment}
Given an SEM $(\G^*, \mB^*, \Set{\var_i})$ with $\G^* \in \Gf_{p,d}$ satisfying Assumption \ref{ass:variance_condition}
such that $(\nicefrac{\Exp{}{N_i}}{\sigma_i})^{4m} \leq K_m, \, \forall i \in [p]$, where $m$ 
is a positive integer and $K_m \in \R^{+}$ is a constant. If the regularization parameter and number of samples
satisfy the following conditions:
\begin{align*}
\lambda_n \geq M C_2 \left(\frac{p^2}{n^m \delta}\right)^{\nicefrac{1}{2m}}, &&
n \geq \frac{(cC_2 4 M^2(1 + \Bmax)\varmax)^2}{\varepsilon^2} \left(\frac{p^2}{\delta}\right)^{\nicefrac{1}{m}},
\end{align*}
then $\Abs{\mB^* - \mhB}_{\infty} \leq \varepsilon$ with probability at least $1 - \delta$,
where $C_2 = 2(\max_i \mSig^*_{i,i})(C_m(C_m(K_m + 1) + 1))^{\nicefrac{1}{2m}}$, 
$C_m$ is a constant that depends only on $m$, $c$ is defined in Theorem \ref{thm:main_finite_sample},
and $M$ is given by \eqref{eq:constant_M}. Further, thresholding $\mhB$ at the level $\varepsilon$
we get that $\Sp(\mhB) = \Sp(\mB^*)$ and $\Gh = \G^*$.
\end{theorem}

\section{Learning SEMs with known error variances}
In this section we focus our attention on the problem of learning
SEMs when the error variances are known upto a constant factor. 
We will consider SEMs $(\G, \mB, \Set{\alpha \var_i})$ where $\Set{\var_i}_{i=1}^p$ are known (to the learner)
and $\alpha > 0$ is some unknown constant.
Identifiability of this class of SEMs was proved by \cite{loh_high-dimensional_2013} under a
\emph{faithfulness} assumption. However, we will merely assume that $(\G, \mB, \Set{\alpha \var_i})$
is causal minimal, i.e., $\Sp(\mB) = \Es$ --- this ensures that the distribution $\Pf(X)$ defined by 
the SEM is causal minimal to the DAG $\G = ([p], \Es)$.
An immediate consequence of Proposition \ref{prop:precision} is 
the following observation about terminal vertices:
\begin{proposition}
\label{prop:terminal_vertex_known_var}
Let $(\G, \mB, \Set{\alpha \var_i})$ be an SEM over $X$ with precision matrix $\mOmg$, $\Set{\var_i}_{i=1}^p$ known and $\alpha > 0$
is some unknown constant. Then, $i$ is a terminal vertex in $\G$ if and only if $i \in \argmin \diag(\mOmg \hadprod \mD)$,
where $\mD = \Diag(\var_1, \ldots, \var_p)$.
\end{proposition}

Thus, when the error variances are known upto a constant factor, Algorithm \ref{alg:population} can be used to learn SEMs,
under the assumption of causal minimality, by setting $\mD = \Diag(\var_1, \ldots, \var_p)$.
Consequently, we have the following result about learning SEMs with known error variances:
\begin{theorem}
\label{thm:main_known_error}
Let $(\G, \mB, \Set{\alpha \var_i})$ be an SEM over $X$, with precision matrix $\mOmg$ and $\Set{\var_i}_{i=1}^p$ known. 
Then, if $(\G, \mB, \Set{\alpha \var_i})$ is causal minimal and given $\mOmg,\,   
\mD = \Diag(\var_1, \ldots, \var_p)$ as input, Algorithm \ref{alg:population}
returns a unique $(\Gh, \mhB)$ such that $\Gh = \G$ and $\mhB = \mB$.
\end{theorem}

\paragraph{Misspecified error variances.}
Our algorithm can also be used to learn SEMs with misspecified error variances as considered by \cite{loh_high-dimensional_2013}.
For instance, if the true SEM is $(\G, \mB, \Set{\var_i})$ while the diagonal matrix passed to Algorithm \ref{alg:population}
is $\mD = \Diag((\sigma'_1)^2, \ldots, (\sigma'_p)^2)$, then it is straightforward to verify that the
following condition is sufficient to ensure that Algorithm \ref{alg:population} still recovers the structure and parameters of the SEM correctly: 
\begin{align*}
 \sum_{\mcp{l \in \Chi{j}{\G[m, \tau]}}} B^2_{l,j}  > \frac{\alphamax}{\alphamin} - 1, &&
(\forall j \in \Vs[m, \tau] \Land \Chi{j}{\G[m, \tau]} \neq \varnothing,  m \in [p], \tau \in \Ts),
\end{align*}
where $\alphamax \defeq \max \Set{\nicefrac{(\sigma'_i)^2}{\sigma_i^2} \mid i \in [p]}$ (similarly $\alphamin$).
Next, we obtain statistical guarantees for our algorithm for learning SEMs with known error variances.

\subsection{Statistical guarantees for estimation}
In order to learn SEMs with known error variances from a finite number of samples, we make the following assumptions:
\begin{assumption}
\label{ass:finite_sample_known_varaince}
Given an SEM $(\G, \mB, \Set{\alpha \var_i})$ with precision matrix $\mOmg$ and $\Set{\var_i}_{i=1}^p$ known,
let $\mOmg_{(m, \tau)}$ denote the inverse covariance matrix over $X_{\Vs[m, \tau]}$.
Then, 
\begin{enumerate}[(i)]
\item $\forall i \in \Vs[m, \tau], m \in [p]$, and $\tau \in \Ts_{\G}$, 
such that $\Chi{i}{\G[m,\tau]} \neq \varnothing$:
\begin{align*}
\sum_{l \in \Chi{i}{\G[m, \tau]}} \left(\frac{\var_i}{\var_l}\right) B_{l,i}^2 > 8 \alpha M \lambda_n,
\end{align*}
\item $\min \Set{\Abs{(\Omega_{(m, \tau)})_{i,j}} \mid (\Omega_{(m, \tau)})_{i,j} \neq 0, (i, j) \in \Vs[m, \tau] \times \Vs[m, \tau],
 m \in [p], \tau \in \Ts_{\G}} > 4 M \lambda_n$,
\item for all $i \in [p]$, $\var_i \in o(\nicefrac{1}{4 \alpha M \lambda_n})$.
\end{enumerate}
\end{assumption}
Using CLIME to estimate and update the precision matrix, it is easy to verify that Theorems \ref{thm:sub_gaussian}
and \ref{thm:bounded_moment} hold for SEMs with known error variances satisfying Assumption \ref{ass:finite_sample_known_varaince},
with $\varmax$ and $\varmin$ replaced by $\alpha \varmax$ and $\alpha \varmin$, respectively. 
Thus, given a data set of $n$ samples drawn from an SEM satisfying Assumption \ref{ass:finite_sample_known_varaince}, 
with autoregression matrix $\mB^*$ and DAG structure $\G^* = ([p], \Es^*)$, we have the following results
about sub-Gaussian and bounded-moment noise:
\begin{remark}
For sub-Gaussian noise, if $\lambda_n = \BigOm{\frac{d^2}{\sqrt{n}} \sqrt{\log (\frac{p}{\sqrt{\delta}})}}$, and
the number of samples $n = \BigOm{\frac{d^8}{\varepsilon^2} \log (\frac{p}{\sqrt{\delta}}) }$,
then Algorithm \ref{alg:population} with $\mD = \Diag(\var_1, \ldots, \var_p)$ returns an estimator $\mhB$ such that
$\Abs{\mhB - \mB^*}_{\infty} \leq \varepsilon$, with probability at least $1 - \delta$.
Thresholding $\mhB$ at the level $\varepsilon$, we have $\Sp(\mhB) = \Es^*$.
\end{remark}
\begin{remark}
For noise with bounded $(4m)$-th moment, with $m$ being a positive integer, 
if the regularization parameter $\lambda_n = \BigOm{\frac{d^2}{\sqrt{n}} (\frac{p}{\sqrt{\delta}})^{\nicefrac{1}{m}} }$, and
the number of samples $n = \BigOm{\frac{d^8}{\varepsilon^2} (\frac{p^2}{\delta})^{\nicefrac{1}{m}}}$,
then Algorithm \ref{alg:population} with $\mD = \Diag(\var_1, \ldots, \var_p)$ returns an estimator $\mhB$ such that
$\Abs{\mhB - \mB^*}_{\infty} \leq \varepsilon$, with probability at least $1 - \delta$.
Thresholding $\mhB$ at the level $\varepsilon$, we have $\Sp(\mhB) = \Es^*$.
\end{remark}
The above remarks use the fact that $M = \BigO{d^2}$, which follows from Proposition \ref{prop:max_support_size} given in Appendix.

\section{Computational complexity}
In the population setting, i.e., given the true precision matrix, our
algorithm can be implemented by storing the diagonal of the precision matrix separately
and sorting it once which takes $\BigO{p \log p}$ time. In each iteration, updating the
precision matrix in line \ref{line:pop_update} takes $\BigO{d^2}$ time since
$\mOmg_{*,i}$ and $\mOmg_{i,*}$ are $d$-sparse. Updating the diagonal takes $\BigO{d \log p}$ time,
while searching for the minimum diagonal element takes $\BigO{\log p}$ time.
Therefore, Algorithm \ref{alg:population} computes the $\mhB$ matrix in $\BigO{p(d^2 + d \log p)}$ time.
In the population setting, the computational complexity of \cite{loh_high-dimensional_2013}'s algorithm is
$\BigO{p 2^{2(w+1)(w+d)}}$, where $w$ is the tree-width of the DAG structure of the true SEM. Note
that the population version of our algorithm can still be used in the finite sample setting if the
precision matrix is estimated accurately enough.

In the finite sample setting, the computational complexity of our algorithm is dominated by the steps for estimating
and updating the precision matrix --- the latter depends on how well the sparsity pattern
of the precision matrix is estimated. First, we analyze the computational
complexity of our algorithm assuming exact support recovery, then we analyze the worst-case performance
of our algorithm without assuming sparsity of the estimated precision matrix. 
Estimating the precision matrix can be done by solving $p$ linear programs in $2p$-dimension and with $4p$ constraints.
The smoothed complexity of this step is $\BigO{p^3 \log (\nicefrac{p}{\sigma})}$ when using interior point
LP solvers \cite{dunagan_smoothed_2011}, where $\sigma^2$ is variance of the Gaussian perturbations 
\footnote{The worst-case complexity of interior point methods for solving LPs is $\BigO{p^3 L}$
where $L$ `` is a parameter measuring the precision needed to perform the arithmetic operations exactly'' 
and grows as $\BigOm{p}$ \cite{spielman_smoothed_2003}. However, interior-point methods work much
more efficiently in practice and have an average complexity of $\BigO{p^3 \log p}$ 
(see \cite{spielman_smoothed_2003} and the references therein).}.
Next observe that $\Abs{\mtOmg - \mhOmg}_{\infty} \leq \Abs{\mB^* - \mhB}_{\infty} \leq \varepsilon$. 
By thresholding $\mhOmg$ at the level $\varepsilon$, each time the precision matrix is updated, 
we can ensure exact support recovery in each iteration. Thus, in the $\Update$ function
$\hPar{i}{} = \Par{i}{\G^*}$ and $\Abs{\Sh_j} \leq d^2 \leq p$. Therefore, the $\Update$ function
takes $\BigO{d^7 \log (\nicefrac{d}{\sigma})}$ operations, leading to an overall complexity of $\BigOT{p^3 + pd^7}$.
In the worst case, i.e., without any thresholding, $\mhOmg$ can be dense. Therefore, the $\Update$ function
might re-estimate the full precision matrix over $p - t$ variables in iteration $t$, which takes
$\BigO{(p - t)^4 \log (\nicefrac{(p - t)}{\sigma})}$ operations, leading to an overall complexity of $\BigOT{p^5}$.
Thus, in the finite sample setting the complexity of our algorithm is between $\BigOT{p^3 + pd^7}$ and $\BigOT{p^5}$.
Note that \cite{loh_high-dimensional_2013}'s 
analysis of the computational complexity of their algorithm assumes perfect support recovery of the precision matrix.
In this regime, the computational complexity of their method is $\BigO{p 2^{2(w+1)(w+d)} + p^3}$, including
the step to estimate the precision matrix using graphical Lasso \cite{friedman_sparse_2008}, where $w$ is the tree-width
of the true DAG. However, without thresholding the output of graphical Lasso can be dense leading to a worst-case
computational complexity that is exponential in $p$.
\section{Appendix}
\begin{appendix}
\section{Detailed Proofs}
\begin{proof}[Proof of Proposition \ref{prop:suff_cond_identifiability}]
When $\var_i = \var$ for all $i \in [p]$, then \eqref{eq:identifiability_condition}
reduces to:
\begin{align*}
\sum_{l \in \Chi{j}{\G[m, \tau]}} \frac{B_{l,j}^2}{\sigma_l^2} > 0,
\end{align*}
which holds trivially by causal minimality since $B_{l,j}^2 > 0$ for $(l, j) \in \Es$. This proves part (i).

Now under (ii), $\nicefrac{1}{\var_i} - \nicefrac{1}{\var_j} < 1\, , \forall i, j \in [p]$.
Also,  $\nicefrac{B_{l,j}^2}{\sigma_l^2} \geq 1$ for all $(l, j) \in \Es$. Thus \eqref{eq:identifiability_condition}
is satisfied.
\end{proof}
\begin{proof}[Proof of Lemma \ref{lemma:necessary}]
Consider the following two SEMs
over three nodes, where the noise variances are shown within braces below each node,
and the edge weights are shown on the edges.
\begin{center}
\includegraphics[width=0.8\linewidth]{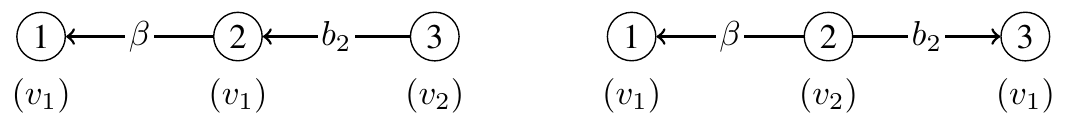}
\end{center}
Both the SEMs make the following conditional independence assertion: $X_1 \independent X_3 \mid X_2$,
and are therefore Markov and causal minimal to $\Pf(X)$. Set $b_2 = \sqrt{1 - \frac{v_1}{v_2}}$.
Then using the formulas
derived in Proposition \ref{prop:precision} it can be verified that the precision matrix for
both the SEMs is:
\begin{align}
\mOmg = \frac{1}{v_1} \times
\matrx{
1 && -\beta && 0 \\
-\beta && 1 + \beta^2 && -b_2 \\
0 && -b_2 && 1
}.   \label{eq:precision_mat_example}
\end{align}
The SEM on the left does not satisfy Assumption \ref{ass:identifiability_condition}
because vertex $3$ is a non-terminal vertex but $3 \in \argmin(\mOmg)$.
The SEM on the right does not satisfy Assumption \ref{ass:identifiability_condition} because
after the vertex $1$ is removed we have that vertex $3$ is a non-terminal vertex but
satisfies $3 \in \argmin(\mOmg_{(-1)})$, where $\mOmg_{(-1)}$ is the precision matrix
over vertices $\Set{2,3}$.

Now we construct the subset $\Gfr_{p,d}$ with $p = 3k$ for $k = 1,2,\ldots$, as follows.
We randomly set the DAG structure over nodes $(3i - 1), (3i)$ and $(3i + 1)$ to one
of the two configurations shown in the above figure. Therefore we have, $\Abs{\Gfr_{p,d}} = 2^{\nicefrac{(p-1)}{3}}$.
We generate matrices $\mB(\beta)$
and $\mD(v_1, v_2)$ as prescribed. The precision matrix block over the nodes $(3i - 1), (3i)$,
and $(3i + 1)$, for $i \in [\nicefrac{(p-1)}{3}]$, is given by \eqref{eq:precision_mat_example},
and all the other entries of the precision matrix are zeros.
This proves our claim. 

While the above constructions constructs a family of disconnected DAGs, with $d = 1$,
it is easy to come up with subsets of DAGs that are connected and still satisfy the statement 
of the lemma. One such construction is shown below where $d = \nicefrac{(p-1)}{3}$.
The entries of the first row (and also the first column) of the precision matrix, 
for $i \in [\nicefrac{(p-1)}{3}]$, are as follows:
\begin{align*}
\Omega_{1,1} = \frac{1}{v_0} + \frac{(p - 1) b_0^2}{3 v_1},\,
\Omega_{1, 3i-1} = -\frac{b_0}{v_1}, \,
\Omega_{1, 3i} = \frac{b_0 \beta}{v_1}.
\end{align*}
\includegraphics[width=\linewidth]{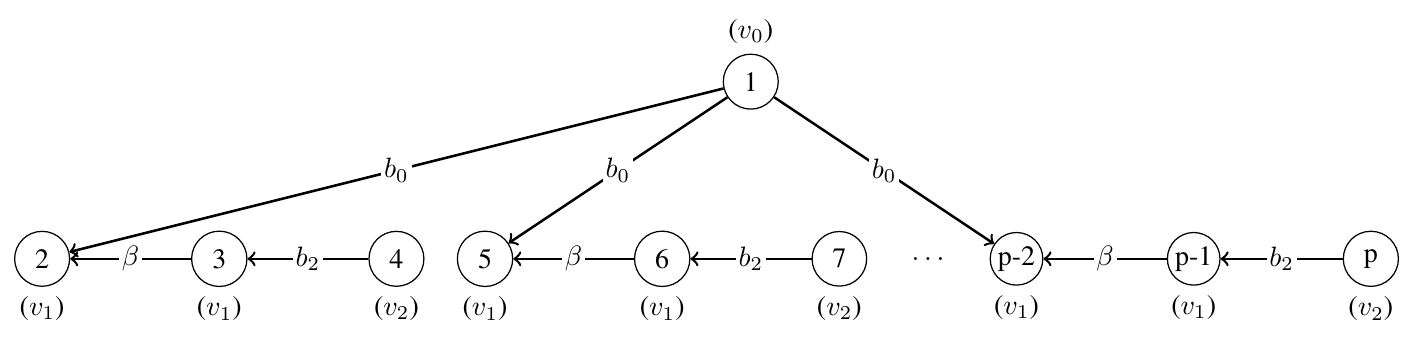}
As shown before, each triplet of nodes $(3i - 1) \leftarrow (3i) \leftarrow (3i + 1)$,
for $i \in [\nicefrac{(p-1)}{3}]$, can be oriented as $(3i - 1) \leftarrow (3i) \rightarrow (3i + 1)$
without changing the block of the precision matrix over the nodes $(3i - 1), (3i)$ and $(3i + 1)$,
and the entries $\mOmg_{1,*}$ or $\mOmg_{*,1}$.  
\end{proof}

\begin{proof}[Proof of Proposition \ref{prop:precision}]
From \eqref{eq:sem_vect} we have that $(\mI - \mB) X = N$, and since
$(\mI - \mB)$ is invertible, $X = (\mI - \mB)^{-1} N$. Therefore:
\begin{align*}
\mSig = \Exp{}{X X^T} = \Exp{}{(\mI - \mB)^{-1} N N^T (\mI - \mB)^{-T}} 
	= (\mI - \mB)^{-1} \mD (\mI - \mB)^{-T}.
\end{align*}
From which it follows that $\mOmg = (\mI - \mB)^T \mD^{-1} (\mI - \mB)$, where 
$\mD^{-1} = \Diag(\nicefrac{1}{\var_1}, \ldots, \nicefrac{1}{\var_p})$.
From this the result for the entries of the precision matrix follows by sparsity pattern of $\mB$.
\end{proof}
\begin{proof}[Proof of Proposition \ref{prop:terminal_vertex}]
From \eqref{eq:precision_mat} we have that for a terminal vertex $i$,
$\Omega_{i,i} = \nicefrac{1}{\var_i}$, while for a non-terminal vertex $j$,
$\Omega_{j,j} = \nicefrac{1}{\var_j} + \sum_{l \in \Chi{j}{}} \nicefrac{B_{l,j}^2}{\var_l}$.
Therefore, by Assumption \ref{ass:identifiability_condition} we have that for all non-terminal
vertices $j$ and terminal vertices $i$, $\Omega_{j,j} > \Omega_{i,i}$.

Now since every DAG has at least one terminal vertex, if $i \in \argmin(\diag(\mOmg))$, then
once again by Assumption \ref{ass:identifiability_condition}, we have  that $i$ must be a terminal vertex.
\end{proof}
\begin{proof}[Proof of Lemma \ref{lemma:schur_update}]
First note that since $i$ is a terminal vertex, the autoregression matrix over $X_{\mi}$ 
is simply $\mB_{\mi, \mi}$. Therefore, denoting $\mD' \defeq \Diag(\var_1, \ldots, \var_{i-1}, \var_{i+1}, \var_p)$
and by Proposition \ref{prop:precision} we have:
\begin{align*}
\mOmg_{(\mi)} &= (\mI - \mB_{\mi, \mi})^T (\mD')^{-1} (\mI - \mB_{\mi, \mi}) 
	= \sum_{j \in \mi} \frac{1}{\var_j} ((\ve_j)_{\mi} - \mB^T_{j,\mi}) ((\ve_j)^T_{\mi} - \mB_{\mi,j}) \\
	&= \sum_{j \in [p]} \frac{1}{\var_j} \left((\ve_j - \mB^T_{j, *}) (\ve^T_j, - \mB_{j, *})\right)_{\mi, \mi} 
		- \frac{1}{\var_i} \left((\ve_i - \mB^T_{i, *}) (\ve^T_i - \mB_{i ,*})\right)_{\mi, \mi} \\
	&= \mOmg_{\mi, \mi} - \frac{1}{\var_i} \left(\mB^T_{i, \mi} \mB_{i, \mi}\right)
	= \mOmg_{\mi, \mi} - \Omega_{i,i} \frac{\mOmg_{i,\mi}^T}{\Omega_{i,i}} \frac{\mOmg_{i,\mi}}{\Omega_{i,i}}
	= \mOmg_{\mi, \mi} - \frac{1}{\Omega_{i,i}} \mOmg_{\mi,i} \mOmg_{i,\mi},
\end{align*}
where in the last line we used the fact that for a terminal vertex $\Omega_{i,i} = \nicefrac{1}{\var_i}$
(Proposition \ref{prop:terminal_vertex}), 
and $\mB_{i,\mi} = -\nicefrac{\mOmg_{i,\mi}}{\Omega_{i,i}}$ (Proposition \ref{prop:weights}).
\end{proof}
\begin{proof}[Proof of Lemma \ref{lemma:precision_minus_i}]
First consider the case when $j \notin \Par{i}{\G}$. Then, for any $k \in [p] \setminus \Set{i,j}$,
$i \notin (\Chi{j}{\G} \intersection \Chi{k}{\G})$. Therefore, by Proposition \ref{prop:precision},
$(\Omega_{(\mi)})_{j,k} = \Omega_{j,k}$, and by symmetry of the precision matrix $(\Omega_{(\mi)})_{k,j} = \Omega_{k,j}$.
Thus, we have that for any $(j,k)$ if at least one of $\Set{j,k}$ is not in $\Par{i}{\G}$, 
then $(\Omega_{(\mi)})_{j,k} = \Omega_{j,k}$, which proves our first claim.
Thus, the only remaining case to consider is when both $j, k \in \Par{i}{\G}$. 
The are two ways is which the set $\Sp((\mOmg_{(\mi)})_{j,*})$ can be larger than the set $\Sp(\mOmg_{j,*})$,
i.e., the support set of the $j$-th node can increase after deleting the terminal node $i$. 
The first being when
$j,k \in \Par{i}{\G}$ and either $(j,k) \in \Es$ or $(k,j) \in \Es$ but $\Omega_{j,k} = 0$, in which 
case we have:
\begin{align*}
\sum_{l \in \Chi{j}{} \intersection \Chi{k}{}} \nicefrac{(B_{l,j} B_{l,k})}{\var_l} = 
\nicefrac{B_{j,k}}{\var_j} + \nicefrac{B_{k,j}}{\var_k}.
\end{align*}
Then, after removing the terminal node $i$, we have 
\begin{align*}
(\Omega_{(\mi)})_{j,k} = - \nicefrac{B_{j,k}}{\var_j} - \nicefrac{B_{k,j}}{\var_k} +
	\sum_{\mcp{l \in (\Chi{j}{} \intersection \Chi{k}{} \setminus \Set{i})}} \nicefrac{(B_{l,j} B_{l,k})}{\var_l} \neq 0.
\end{align*} 
The other case is when $j,k \in \Par{i}{\G}$, $(j,k) \notin \Es$, $(k,j) \notin \Es$ 
but $\Omega_{j,k} = 0$, in which case we have:
\begin{align*}
\sum_{l \in \Chi{j}{} \intersection \Chi{k}{}} \nicefrac{(B_{l,j} B_{l,k})}{\var_l} = 0.
\end{align*}
Therefore, after removing the terminal node we have:
\begin{align*}
(\Omega_{(\mi)})_{j,k} = 
	\sum_{\mcp{l \in (\Chi{j}{} \intersection \Chi{k}{} \setminus \Set{i})}} \nicefrac{(B_{l,j} B_{l,k})}{\var_l} \neq 0.
\end{align*} 
Thus, $\Sp((\mOmg_{(\mi)})_{j,*}) \subseteq (\Sp(\mOmg_{j,*}) \setminus \Set{i}) \union \Par{i}{\G}$.
\end{proof}
\begin{proof}[Proof of Theorem \ref{thm:main_unknown_error}]
Let $i_t$ be the terminal vertex identified in iteration $t$, $\cI_t \defeq \Set{i_1, \ldots, i_t}$
and $\cR_t \defeq [p] \setminus \cI_t$. Let $\mOmg_{(i)}$ be the precision matrix after iteration $i$.
The correctness of the algorithm follows from the following loop invariants:
\begin{enumerate}[(i)]
\item By Lemma \ref{lemma:schur_update} we have that, $(\mOmg_{(t)})_{\cR_t, \cR_t}$ is the correct precision matrix over $X_{\cR_t}$.
\item The algorithm identifies a correct terminal vertex in iteration $t$, since 
$(\mOmg_{(t-1)})_{\cR_{t-1}, \cR_{t-1}}$ is the correct precision matrix over $X_{\cR_{t-1}}$, 
the SEM over $X_{\cR_{t-1}}$ satisfies Assumption \ref{ass:identifiability_condition} by definition,
and $\forall i \in \cI_{t-1}, \, \Omega_{i,i} = \infty$.
\item By proposition \ref{prop:terminal_vertex} we have that at the end of round $t$, the sub-matrix $\mB_{\cI_t, *}$ has
	been correctly set and that $\forall i \in \cI_t, \, \Par{i}{\G} = \Sp(\mB_{i, *})$.
\end{enumerate}
To see that the algorithm returns a unique autoregression matrix $\mhB$, consider the following.
If at iteration $t$ there is a unique minimizer of $\diag(\mOmg_{(t - 1)})$, which implies a single
terminal vertex, then the algorithm selects it and the incoming edge weights of the node is uniquely determined. 
While, in iteration $t$ if 
there are multiple terminal vertices, leading to multiple minimizers of $\diag(\mOmg_{(t - 1)})$,
then the order in which they are eliminated does not matter. Or in other words, once
a vertex becomes a terminal vertex, for instance after deletion of its children, its edge weights 
do not change. To see this, assume that there are two terminal vertices, $i$ and $j$ after iteration $t - 1$.
Then $i$ and $j$ are not in each other's parent sets. Therefore, if node $i$ is eliminated in iteration $t$,
then by Lemma \ref{lemma:precision_minus_i} we have that $(\Omega_{(t)})_{j,k} = (\Omega_{(t - 1)})_{j,k},\,
\forall k \in \Par{j}{\G}$. Hence, we have that $\mB$ is the unique autoregression matrix returned by
the algorithm. 
\end{proof}
\begin{proof}[Proof of Lemma \ref{lemma:clime_update}]
Let $\mOmg_{(\mi)} = (\vomg_{j})_{j \in \mi}$ be the true precision matrix over $X_{\mi}$ and let $\mhOmg' = (\vomg'_j)_{j \in [p]}$ 
be the matrix returned by the function \Update{}. The estimator $\mhOmg_{(\mi)} = (\vhomg_{j})_{j \in \mi}$ of $\mOmg_{(\mi)}$
can be obtained by solving \eqref{eq:omegahat_colwise} using $\mSig^n_{\mi, \mi}$. By Lemma \ref{lemma:clime},
and the facts that $\Abs{\mSig^n_{\mi, \mi} - \mSig_{\mi, \mi}}_{\infty} \leq \Abs{\mSig^n - \mSig}_{\infty}$ and
$\NormI{\mOmg_{(i)}} \leq M$, we have that $\Abs{\mOmg_{(\mi)} - \mhOmg_{(\mi)}} \leq 4 M \lambda_n$.
Since $i$ is a terminal vertex, by Proposition \ref{prop:weights} we have $\Par{i}{\G} = \Sp(\mOmg_{i,*}) \setminus \Set{i}$.
Further, since $\Sp(\mOmg_{j,*}) \subseteq \Sp(\mhOmg_{j,*})$, $\forall j \in [p]$, we have by Assumption \ref{ass:variance_condition} (ii) that, 
$\Par{i}{\G} \subseteq \hPar{i}{} = \Sp(\mhOmg_{i,*}) \setminus \Set{i} \subseteq \Sh$. 
By Lemma \ref{lemma:precision_minus_i} and Assumption \ref{ass:variance_condition} (ii)
 we have that $\forall j \in \Sh_j$,
$\Sp(\vomg_j) \subseteq \Sp\left(\mOmg_{j, *} \setminus \Set{i} \right) \union \Par{i}{\G} 
\subseteq \Sp\left(\mhOmg_{j, *} \setminus \Set{i} \right) \union \hPar{i}{} \defeq \Sh_j$. 
Or in other words we have $\left(\mOmg_{(i)} \right)_{j, \Sh_j^c} = \left(\mOmg_{(i)} \right)_{\Sh_j^c, j} = \vect{0}$.
Now for $j \in \mi$ we set $(\vomg'_j)_{\Sh_j} = \vbomg_j$ and $(\vomg'_j)_{\Sh_j^c} = \vect{0}$, where 
$\vbomg_j$ is obtained by solving:
\begin{align*}
\begin{array}{ll}
\underset{\vomg \in \R^{\Abs{\Sh_j}}}{\argmin} & \NormI{\vomg}, \\
\text{sub. to} & \aAbs{\mSig^n_{k, \Sh_j} \vomg} \leq \lambda_n, \, \forall k \notin \Set{i,j}, \\
 & \aAbs{\mSig^n_{j, \Sh_j}\vomg - 1} \leq \lambda_n.
\end{array}
\end{align*}
Since $\vbomg_j$ is a solution to the above linear program, we have that $\Abs{\mSig^n_{\mi, \mi} \vomg'_j - \ve_j} \leq \lambda_n$
and $\NormI{\vomg'_j} \leq \NormI{\vhomg_j}$. Therefore, $\Abs{\mOmg_{(\mi)} - \mhOmg'_{\mi, \mi}} \leq 4 M \lambda_n$.
Moreover, by Assumption \ref{ass:variance_condition} (ii), and the fact that $\mhOmg'_{i,*} = \mhOmg'_{*,i} = \vect{0}$, we get:
$\Sp(\mOmg_{(\mi)}) \subseteq \Sp(\mhOmg')$.
\end{proof}

\begin{proof}[Proof of Theorem \ref{thm:main_finite_sample}]
Let $i_t$ denote the terminal vertex identified in iteration $t$ and let $\cI_t \defeq \Set{i_1, \ldots, i_t}$.
Let $\cR_t \defeq [p] \setminus \cI_t$ denote the vertices remaining after iteration $t$.
Let $\mhOmg_{(t)}$ denote the precision matrix at the end of iteration $t$,
$\mhOmg_{(\cR_t)} \defeq (\mhOmg_{(t)})_{\cR_t, \cR_t}$, and $\mtOmg_{(\cR_t)}$ be
the true precision matrix over $X_{\cR_t}$. Since $\NormI{\mtOmg} \leq M$, where $M$
is defined in \eqref{eq:constant_M}, we have that 
$\lambda_n \geq M \Abs{\mSig^n - \mSig^*}_{\infty} \geq \NormI{\mtOmg} \Abs{\mSig^n - \mSig^*}_{\infty}$.
Therefore, by Lemma \ref{lemma:clime} and Assumption \ref{ass:variance_condition} (ii), we have that 
$\aAbs{\mhOmg_{(\cR_0)} - \mtOmg_{(\cR_0)}}_{\infty} = \Abs{\mhOmg - \mtOmg}_{\infty} \leq 4 M \lambda_n$,
and $\Sp(\mtOmg_{(\cR_0)}) \subseteq \Sp(\mhOmg_0)$. Therefore, by Assumption \ref{ass:variance_condition}
we have that the Algorithm \ref{alg:population} identifies the correct terminal vertex in iteration $1$.
Therefore, by Lemma \ref{lemma:clime_update} we have that $\aAbs{\mtOmg_{(\cR_{t_1})} - \mhOmg_{(\cR_{t_1})}} \leq 4 M \lambda_n$
and $\Sp(\mtOmg_{(\cR_{t_1})}) \subseteq \mhOmg_{(t_1)}$.

Let $\mE = (\varepsilon_{i,j})$, where $\varepsilon_{i,j} = \Omega^*_{i,j} - \hOmega_{i,j}$.
To simplify notation in this paragraph, we will denote the $i_1$ vertex by simply $i$.
Then, for any $j \neq i$, we have that
\begin{align*}
\Abs{\hB_{i,j} - B^*_{i,j}} 
	&= \aAbs{\frac{\hOmega_{i,j}}{\hOmega_{i,i}} - \frac{\Omega^*_{i,j}}{\Omega^*_{i,i}}}
	= \aAbs{\frac{\Omega^*_{ii}(\Omega^*_{i,j} - \varepsilon_{i,j}) - 
	              (\Omega^*_{i,i} - \varepsilon_{i,i}) \Omega^*_{i,j}}{(\Omega^*_{i,i} - \varepsilon_{i,i}) \Omega^*_{i,i}}} \\
	&= \aAbs{\frac{\Omega^*_{i,i} \varepsilon_{i,j} - \Omega^*_{i,j} \varepsilon_{i,i}}
		{(\Omega^*_{i,i} - \varepsilon_{i,i}) \Omega^*_{i,i}}}
	 = \aAbs{ \frac{\varepsilon_{i,i} - \var_i \Omega^*_{i,j} \varepsilon_{i,i} }{\nicefrac{1}{\var_i} - \varepsilon_{i,i}} } \\
	&= \aAbs{\frac{\varepsilon_{i,i} - B^*_{i,j} \varepsilon_{i,i} }{\nicefrac{1}{\var_i} - \varepsilon_{i,i}}} \\
	&\leq \frac{4 M \lambda_n(1 + \Abs{B^*_{i,j}})}{\Abs{\nicefrac{1}{\var_i} - \varepsilon_{i,i}}} 
	\leq 4cM (1 + \Abs{B^*_{i,j}}) \var_i \lambda_n,
\end{align*}
where the second and third lines follow from the fact that $i$ is a terminal vertex and
therefore, $\Omega^*_{i,i} = \nicefrac{1}{\var_i}$ and $\Omega_{i,j} = \nicefrac{-B_{i,j}}{\var_i}$. 
Therefore, we have that $\Abs{\mB^*_{i_1, *} - \mhB_{i_1,*}}_{\infty} = 4 c M (1 + \Bmax) \varmax \lambda_n$.

Next, assume that the algorithm correctly identifies terminal vertices upto round $t$. Then
$\Abs{\mhOmg_{(\cR_t)} - \mtOmg_{(\cR_t)}}_{\infty} \leq 4 M \lambda_n$, 
$\Sp(\mtOmg_{(\cR_t)}) \subseteq \Sp(\mhOmg_{(t)})$, and 
$\Abs{\mB^*_{\cI_t, \cI_t} - \mhB_{\cI_t, \cI_t}} \leq 4cM (1 + \Bmax) \varmax \lambda_n$.
Therefore, once again by Assumption \ref{ass:variance_condition}, it follows that the algorithm
identifies the correct terminal vertex in round $t+1$, 
$\Abs{\mhOmg_{(\cR_{t+1})} - \mtOmg_{(\cR_{t+1})}}_{\infty} \leq 4 M \lambda_n$, 
$\Sp(\mtOmg_{(\cR_{t+1})}) \subseteq \Sp(\mhOmg_{(t+1)})$, and 
$\Abs{\mB^*_{\cI_{t+1}, \cI_{t+1}} - \mhB_{\cI_{t+1}, \cI_{t+1}}} \leq 4cM (1 + \Bmax) \varmax \lambda_n$.
Hence, the final claim follows  by induction. The claim that $\Sp(\mB^*) \subseteq \Sp(\mhB)$
follows from the fact that $\Sp(\mtOmg) \subseteq \Sp(\mhOmg)$. Finally, since $\Sp(\mB^*) \subseteq \Sp(\mhB)$
implies that $\Ts_{\Gh} \subseteq \Ts_{\G^*}$.
\end{proof}
\begin{proof}[Proof of Theorem \ref{thm:sub_gaussian}]
Given that the data was generated by the SEM $(\G^*, \mB^*, \Set{\var_i})$, each $X_i$
can be written as follows:
\begin{align*}
X_i = \sum_{\mcp{j \in \Anc{i}{\G^*}}} w_{i,j} N_j,
\end{align*}
for some $w_{i,j} \geq 0$. $N_i$ is sub-Gaussian with parameter $\sigma_i \nu$, $X_i$ is
sub-Gaussian with parameter $\nu \sqrt{\sum_{j \in \Anc{i}{\G^*}} w^2_{i,j} \var_i}$
and $\Sigma^*_{i,i} = \sum_{j \in \Anc{i}{\G^*}} w^2_{i,j} \var_i$. Therefore, it 
follows that $\nicefrac{X_i}{\sqrt{\Sigma^*_{i,i}}}$ is sub-Gaussian with parameter $\nu$.
From Lemma 1 of \cite{ravikumar_high-dimensional_2011} and Theorem \ref{thm:main_finite_sample}
we have that the regularization parameter $\lambda_n$ need to satisfy
the following bound in order to guarantee that $\Abs{\mhB - \mB^*}_{\infty} \leq \varepsilon$:
\begin{align}
M C_1 \sqrt{\frac{2}{n} \log \left(\frac{2p}{\sqrt{\delta}}\right) }
	\leq \lambda_n \leq \frac{\varepsilon}{c4M(1 + \Bmax) \varmax}.
\end{align} 
The above holds in the regime where the number of samples scales as given in the statement of the Theorem.
\end{proof}
\begin{proof}[Proof of Theorem \ref{thm:bounded_moment}]
Given that the data was generated by the SEM $(\G^*, \mB^*, \Set{\var_i})$, each $X_i$
can be written as follows:
\begin{align*}
X_i = \sum_{\mcp{j \in \Anc{i}{\G^*}}} w_{i,j} N_j,
\end{align*}
for some $w_{i,j} \geq 0$.
Now,
\begin{align}
\left(\sqrt{\mSig^*_{i,i}}\right)^{4m} 
	= \left(\sum_{j \in \Anc{i}{\G^*}} w^2_{i,j} \var_i \right)^{2m} \geq \sum_{j \in \Anc{i}{\G^*}} (w_{i,j} \sigma_i)^{4m} 
	\label{eq:proof_bm_1}
\end{align}
Now, by Rosenthal's inequality we have:
\begin{align}
\Exp{}{(X_i)^{4m}} &\leq C_m\left\{\quad \sum_{\mcp{j \in \Anc{i}{\G^*}}} w^{4m}_{i,j} \Exp{}{N_j^{4m}} 
	+ \sum_{\mcp{j \in \Anc{i}{\G^*}}} w^{4m}_{i,j} \Var{}{N_i}^{2m} \right\} \notag \\
	&\leq C_m\left\{\quad \sum_{\mcp{j \in \Anc{i}{\G^*}}} w^{4m}_{i,j} \sigma_i^{4m} K_m 
		+ \sum_{\mcp{j \in \Anc{i}{\G^*}}} w^{4m}_{i,j} \sigma_i^{4m} \right\} \notag \\
	&= C_m(K_m + 1) 	\sum_{\mcp{j \in \Anc{i}{\G^*}}} (w_{i,j}	\sigma_i)^{4m} \label{eq:proof_bm_2}
\end{align}
Combining \eqref{eq:proof_bm_1} and \eqref{eq:proof_bm_2} we have
\begin{align}
\Exp{}{\left( \frac{X_i}{\sqrt{\Sigma^*_{i,i}}} \right)^{4m} } \leq C_m(K_m + 1).
\end{align}
From the above and invoking Lemma 2 of \cite{ravikumar_high-dimensional_2011} we get:
\begin{align}
\Abs{\mSig^n - \mSig^*}_{\infty} < C_2 \left(\frac{p^2}{n^m \delta}\right)^{\nicefrac{1}{2m}}, \label{eq:proof_bm_3}
\end{align}
with probability at least $1 - \delta$.
From Theorem \ref{thm:main_finite_sample} and \eqref{eq:proof_bm_3} we have that the regularization parameter
$\lambda$ should satisfy the following for $\Abs{\mhB - \mB^*}_{\infty} \leq \varepsilon$ to hold:
\begin{align}
M C_2 \left(\frac{p^2}{n^m \delta}\right)^{\nicefrac{1}{2m}} \leq \lambda_n \leq 
	\frac{\varepsilon}{c4M(1 + \Bmax) \varmax}.
\end{align}
The above holds in the regime where the number of samples scales as given in the statement of the Theorem.
\end{proof}
\begin{proposition}
\label{prop:max_support_size}
Let $(\G, \mB, \Set{\var_i})$ be an SEM over $X$ with $\G \in \Gf_{p,d}$ and
 precision matrix $\mOmg$. Then, $\Abs{\Sp(\mOmg_{i,*}) \setminus \Set{i}} \leq d^2,\, \forall i \in [p]$.
\end{proposition}
\begin{proof}[Proof of Proposition \ref{prop:max_support_size}]
For any node $i$, we will define the following set: 
$\Ss_{\G}(i) = \Set{j \in \mi \mid (i,j) \notin \Es \Land (j,i) \notin \Es \Land \Abs{\Omega_{i,j}} \neq 0}$.
Then, from Proposition \ref{prop:precision}, we have: if $j \in \Ss_{\G}(i)$ then 
$\Omega_{i,j} = \sum_{l \in \Chi{i}{} \intersection \Chi{j}{}} \nicefrac{(B_{l,i} B_{l,j})}{\var_l} \neq 0$.
In other words, if $j \in \Ss_{\G}(i)$ then $i$ and $j$ have at least one common child, i.e., 
$\Chi{i}{\G} \intersection \Chi{j}{\G} \neq \varnothing$.
Node $i$ can have at most $d$ children, and each child $k \in \Chi{i}{\G}$ can have at most
$d - 1$ parents other than $i$ making them all members of $\Ss(i)$. Thus, $\Ss(i) \leq d(d - 1)$.
Therefore, we have that $\Sp(\mOmg_{i,*}) \setminus \Set{i} \subseteq \NB{i}{\G} \union \Ss_{\G}(i)$.
Then, using the inclusion-exclusion principle we have that:
\begin{align*}
\Abs{\Sp(\mOmg_{i,*}) \setminus \Set{i}} \leq \Abs{\NB{i}{\G}} + \Abs{\Ss_{\G}(i)} 
	- \Abs{\NB{i}{\G} \intersection \Ss_{\G}(i)} = \Abs{\NB{i}{\G}} + \Abs{\Ss_{\G}(i)} \leq d + d(d-1) = d^2.
\end{align*}
The SEM which achieves the above upper bound is precisely the one constructed in the proof, i.e.,
there exists a node $i$ with exactly $d$ children, each child in turn has $d - 1$ ``other parents''
which are all members of $\Ss_{\G}(i)$.
\end{proof}
\end{appendix}

\bibliographystyle{alpha}
\bibliography{paper}

\begin{thebibliography}{10}

\bibitem{banerjee2008model}
Onureena Banerjee, Laurent~El Ghaoui, and Alexandre d’Aspremont.
\newblock Model selection through sparse maximum likelihood estimation for
  multivariate gaussian or binary data.
\newblock {\em Journal of Machine Learning Research}, 9(Mar):485--516, 2008.

\bibitem{cai_constrained_2011}
Tony Cai, Weidong Liu, and Xi~Luo.
\newblock A {Constrained} {L}1 {Minimization} {Approach} to {Sparse}
  {Precision} {Matrix} {Estimation}.
\newblock {\em Journal of the American Statistical Association},
  106(494):594--607, 2011.

\bibitem{chickering1996learning}
David~Maxwell Chickering.
\newblock Learning bayesian networks is np-complete.
\newblock In {\em Learning from data}, pages 121--130. Springer, 1996.

\bibitem{chickering_optimal_2003}
David~Maxwell Chickering.
\newblock Optimal {Structure} {Identification} with {Greedy} {Search}.
\newblock {\em J. Mach. Learn. Res.}, 3:507--554, March 2003.

\bibitem{dasgupta1999learning}
Sanjoy Dasgupta.
\newblock Learning polytrees.
\newblock In {\em Proceedings of the Fifteenth conference on Uncertainty in
  artificial intelligence}, pages 134--141. Morgan Kaufmann Publishers Inc.,
  1999.

\bibitem{dunagan_smoothed_2011}
John Dunagan, Daniel~A. Spielman, and Shang-Hua Teng.
\newblock Smoothed analysis of condition numbers and complexity implications
  for linear programming.
\newblock {\em Mathematical Programming}, 126(2):315--350, February 2011.

\bibitem{friedman_sparse_2008}
Jerome Friedman, Trevor Hastie, and Robert Tibshirani.
\newblock Sparse inverse covariance estimation with the graphical lasso.
\newblock {\em Biostatistics}, 9(3):432--441, 2008.

\bibitem{ghoshal2016information}
Asish Ghoshal and Jean Honorio.
\newblock {Information-theoretic limits of Bayesian network structure
  learning}.
\newblock In Aarti Singh and Jerry Zhu, editors, {\em Proceedings of the 20th
  International Conference on Artificial Intelligence and Statistics},
  volume~54 of {\em Proceedings of Machine Learning Research}, pages 767--775,
  Fort Lauderdale, FL, USA, 20--22 Apr 2017. PMLR.

\bibitem{hsieh2012divide}
Cho-Jui Hsieh, Arindam Banerjee, Inderjit~S Dhillon, and Pradeep~K Ravikumar.
\newblock A divide-and-conquer method for sparse inverse covariance estimation.
\newblock In {\em Advances in Neural Information Processing Systems}, pages
  2330--2338, 2012.

\bibitem{hsieh_big_2013}
Cho-Jui Hsieh, M\`{a}ty\`{a}s~A Sustik, Inderjit~S Dhillon, Pradeep Ravikumar,
  and Russell Poldrack.
\newblock {BIG} \& {QUIC} : {Sparse} {Inverse} {Covariance} {Estimation} for a
  {Million} {Variables}.
\newblock In {\em Advances in {Neural} {Information} {Processing} {Systems}},
  volume~26, pages 3165--3173, 2013.

\bibitem{jaakkola_learning_2010}
Tommi~S. Jaakkola, David Sontag, Amir Globerson, Marina Meila, and {others}.
\newblock Learning {Bayesian} {Network} {Structure} using {LP} {Relaxations}.
\newblock In {\em {AISTATS}}, pages 358--365, 2010.

\bibitem{johnson2012high}
Christopher~C Johnson, Ali Jalali, and Pradeep Ravikumar.
\newblock High-dimensional sparse inverse covariance estimation using greedy
  methods.
\newblock In {\em AISTATS}, volume~22, pages 574--582, 2012.

\bibitem{kalisch_estimating_2007}
Markus Kalisch and B\"{u}hlmann Peter.
\newblock Estimating {High}-{Dimensional} {Directed} {Acyclic} {Graphs} with
  the {PC}-{Algorithm}.
\newblock {\em Journal of Machine Learning Research}, 8:613--636, 2007.

\bibitem{loh_high-dimensional_2013}
Po-Ling Loh and Peter B\"{u}hlmann.
\newblock High-dimensional learning of linear causal networks via inverse
  covariance estimation.
\newblock {\em arXiv:1311.3492 [math, stat]}, November 2013.
\newblock arXiv: 1311.3492.

\bibitem{park_learning_2017}
Gunwoong Park and Garvesh Raskutti.
\newblock Learning {Quadratic} {Variance} {Function} ({QVF}) {DAG} models via
  {OverDispersion} {Scoring} ({ODS}).
\newblock {\em arXiv:1704.08783 [cs, stat]}, April 2017.
\newblock arXiv: 1704.08783.

\bibitem{Peters2014}
J.~Peters and P.~B{\"{u}}hlmann.
\newblock {Identifiability of Gaussian structural equation models with equal
  error variances}.
\newblock {\em Biometrika}, 101(1):219--228, 2014.

\bibitem{peters_causal_2014}
Jonas Peters, Joris~M Mooij, Dominik Janzing, and Bernhard Sch{\"o}lkopf.
\newblock Causal {Discovery} with {Continuous} {Additive} {Noise} {Models}.
\newblock {\em Journal of Machine Learning Research}, 15(June):2009--2053,
  2014.

\bibitem{ravikumar_high-dimensional_2011}
Pradeep Ravikumar, Martin~J. Wainwright, Garvesh Raskutti, and Bin Yu.
\newblock High-dimensional covariance estimation by minimizing
  $\ell_1$-penalized log-determinant divergence.
\newblock {\em Electronic Journal of Statistics}, 5(0):935--980, 2011.

\bibitem{Robinson1977}
R~W Robinson.
\newblock {Counting unlabeled acyclic digraphs}.
\newblock {\em Combinatorial Mathematics V}, 622:28--43, 1977.

\bibitem{rolfs2012iterative}
Benjamin Rolfs, Bala Rajaratnam, Dominique Guillot, Ian Wong, and Arian Maleki.
\newblock Iterative thresholding algorithm for sparse inverse covariance
  estimation.
\newblock In {\em Advances in Neural Information Processing Systems}, pages
  1574--1582, 2012.

\bibitem{Shimizu2006}
Shohei Shimizu, Patrik~O Hoyer, Aapo Hyv{\"{a}}rinen, and Antti Kerminen.
\newblock {A Linear Non-Gaussian Acyclic Model for Causal Discovery}.
\newblock {\em Journal of Machine Learning Research}, 7:2003--2030, 2006.

\bibitem{shimizu_directlingam:_2011}
Shohei Shimizu, Takanori Inazumi, Yasuhiro Sogawa, Aapo Hyvärinen, Yoshinobu
  Kawahara, Takashi Washio, Patrik~O. Hoyer, and Kenneth Bollen.
\newblock {DirectLiNGAM}: {A} {Direct} {Method} for {Learning} a {Linear}
  {Non}-{Gaussian} {Structural} {Equation} {Model}.
\newblock {\em Journal of Machine Learning Research}, 12(Apr):1225--1248, 2011.

\bibitem{spielman_smoothed_2003}
Daniel~A. Spielman and Shang-Hua Teng.
\newblock Smoothed analysis of termination of linear programming algorithms.
\newblock {\em Mathematical Programming}, 97(1):375--404, 2003.

\bibitem{spirtes2000causation}
Peter Spirtes, Clark~N Glymour, and Richard Scheines.
\newblock {\em Causation, prediction, and search}.
\newblock MIT press, 2000.

\bibitem{tsamardinos2006max}
Ioannis Tsamardinos, Laura~E Brown, and Constantin~F Aliferis.
\newblock The max-min hill-climbing bayesian network structure learning
  algorithm.
\newblock {\em Machine learning}, 65(1):31--78, 2006.

\bibitem{van_de_geer_l0-penalized_2013}
Sara Van De~Geer and Peter B\"{u}hlmann.
\newblock L0-{Penalized} maximum likelihood for sparse directed acyclic graphs.
\newblock {\em Annals of Statistics}, 41(2):536--567, 2013.

\bibitem{yuan2007model}
Ming Yuan and Yi~Lin.
\newblock Model selection and estimation in the gaussian graphical model.
\newblock {\em Biometrika}, 94(1):19--35, 2007.

\bibitem{zhang2002strong}
Jiji Zhang and Peter Spirtes.
\newblock Strong faithfulness and uniform consistency in causal inference.
\newblock In {\em Proceedings of the Nineteenth conference on Uncertainty in
  Artificial Intelligence}, pages 632--639. Morgan Kaufmann Publishers Inc.,
  2002.

\bibitem{Zhang2008}
Jiji Zhang and Peter Spirtes.
\newblock Detection of unfaithfulness and robust causal inference.
\newblock {\em Minds and Machines}, 18(2):239--271, 2008.

\end{thebibliography}

\end{document}